\documentclass{article}

\usepackage{microtype}
\usepackage{graphicx}
\usepackage{subcaption}
\usepackage{booktabs} 
\usepackage[table]{xcolor}
\usepackage{multirow}

\usepackage{amsmath}
\usepackage{amssymb}
\usepackage{mathtools}
\usepackage{amsthm}
\usepackage{algorithm}
\usepackage{algorithmic}

\usepackage{tikz}
\usepackage{pgfplots}
\usepgfplotslibrary{groupplots}
\pgfplotsset{compat=1.18}

\usepackage{hyperref}

\usepackage[accepted]{icml2026}
\usepackage[capitalize,noabbrev]{cleveref}

\theoremstyle{plain}
\newtheorem{theorem}{Theorem}[section]
\newtheorem{proposition}[theorem]{Proposition}

\theoremstyle{definition}

\theoremstyle{remark}
\newtheorem{remark}[theorem]{Remark}

\newcommand{\methodname}{\textbf{\textsc{CoRe}}}
\newcommand{\fullmethodname}{Collaborative Reasoning}

\icmltitlerunning{\methodname: Collaborative Reasoning via Cross Teaching}

\begin{document}

\twocolumn[
\icmltitle{\methodname: Collaborative Reasoning via Cross Teaching}

\begin{icmlauthorlist}
  \icmlauthor{Kshitij Mishra}{mbzuai}
  \icmlauthor{Mirat Aubakirov}{mbzuai}
  \icmlauthor{\texorpdfstring{Martin Tak\'a\v{c}}{Martin Takac}}{mbzuai}
  \icmlauthor{Nils Lukas}{mbzuai}
  \icmlauthor{Salem Lahlou}{mbzuai}
\end{icmlauthorlist}

\icmlaffiliation{mbzuai}{Mohamed bin Zayed University of Artificial Intelligence, Abu Dhabi, United Arab Emirates}
\icmlcorrespondingauthor{Kshitij Mishra}{kshitij.mishra@mbzuai.ac.ae}
\icmlcorrespondingauthor{Salem Lahlou}{salem.lahlou@mbzuai.ac.ae}

\hypersetup{pdfauthor={Kshitij Mishra, Mirat Aubakirov, Martin Takac, Nils Lukas, Salem Lahlou}}

\icmlkeywords{collaborative reasoning, reinforcement learning for LLMs, policy optimization, diversity, ensembles}

\vskip 0.3in
]

\printAffiliationsAndNotice{}   


\begin{abstract}
Large language models exhibit complementary reasoning errors: on the same instance, one model may succeed with a particular decomposition while another fails.
We propose \fullmethodname{} (\methodname), a \emph{training-time} collaboration framework that converts peer success into a learning signal via a \textbf{cross-teaching} protocol.
Each problem is solved in two stages: a \emph{cold} round of independent sampling, followed by a \emph{contexted rescue} round in which failed models receive a hint extracted from a successful peer.
\methodname{} optimizes a combined reward that balances (i) correctness, (ii) a lightweight DPP-inspired diversity term to reduce error overlap, and (iii) an explicit rescue bonus for successful recovery.
We evaluate \methodname{} in low-data regimes across GSM8K, MATH, AIME, and GPQA.
With only 1{,}000 training examples, a pair of small open-source models (\texttt{3B}+\texttt{4B}) reaches \textbf{99.54\%} oracle Team Pass@2 on GSM8K and \textbf{92.08\%} on MATH, compared to \textbf{82.50\%} and \textbf{74.82\%} for single-model training.
On harder datasets, the same pair reaches \textbf{77.34\%} oracle Team Pass@2 on GPQA (trained on 348 examples) and \textbf{79.65\%} on AIME (trained on 792 examples), using a training-time budget of at most 1536 context tokens and 3072 generated tokens.
A three-model non-oracle AIME evaluation further closes most of the oracle-selection gap (81.20\% majority vote vs. 82.60\% oracle Team Pass@2).
Overall, these results show that training-time collaboration can reliably convert model complementarity into large gains without scaling model size.
\end{abstract}

\section{Introduction}

Large language models (LLMs) can solve many multi-step reasoning problems when prompted to produce intermediate steps (chain-of-thought; \citealt{wei2022chain}).
However, even strong models exhibit systematic and complementary failures: different architectures, pretraining mixes, and post-training pipelines lead to distinct error modes on the \emph{same} dataset.
In practice, ensembles and multi-agent prompting exploit this diversity at inference time (e.g., self-consistency \citealt{wang2023selfconsistency}, ranking-and-fusion \citealt{jiang2023llmblender}, and debate-style interactions \citealt{irving2018ai}).
A natural next question is \textit{whether we can push this complementarity \emph{into training} so that models become better collaborators}.

\begin{figure}[t]
  \centering
  \includegraphics[width=\columnwidth]{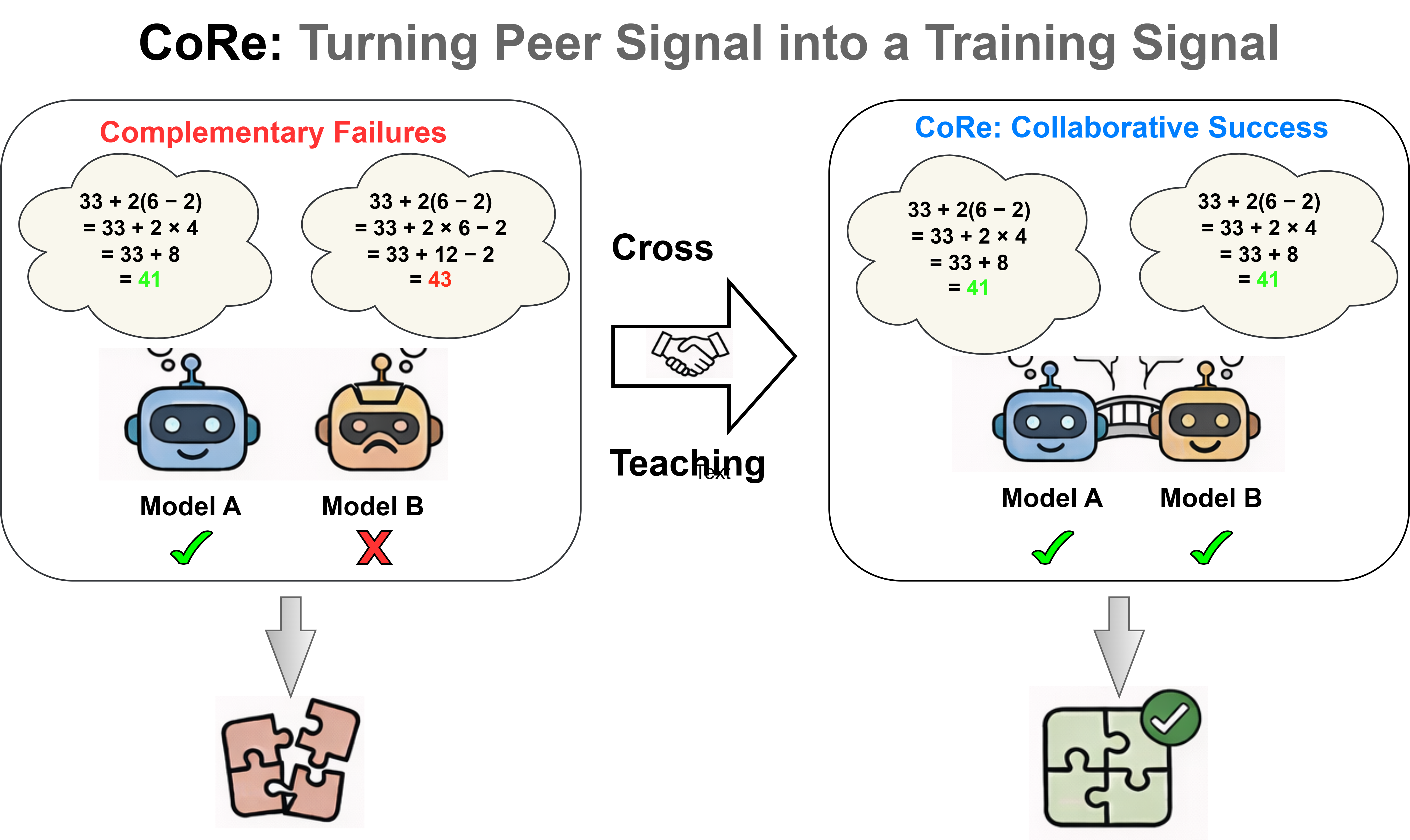}
  \caption{\textbf{\methodname{}}: Collaborative Reasoning via Cross Teaching Training.} 
  \label{fig:framework1}
\end{figure}

Most recent gains in LLM reasoning come from inference-time scaffolding (more samples, search, or multi-agent prompting), which increases compute but does not change the underlying policy \citep{wang2023selfconsistency,yao2023tree,du2024debate}.
These methods often break under \emph{correlated} failure: if a model misses a key insight, additional samples and even additional agents can converge to the same wrong plan; more generally, ensemble gains are limited when member errors are correlated \citep{dietterich2000ensemble,kuncheva2003diversity}.
Training-time collaboration is challenging due to sparse correctness supervision (motivating process-level feedback), multi-model credit assignment, and the risk of diversity collapse through mutual imitation or agreement-style training \citep{lightman2024verify,foerster2018coma,hinton2015distilling,zhang2018deep}.
This motivates a reinforcement learning (RL) view: we must jointly optimize (i) exploitation for correctness, (ii) exploration for complementary traces that reduce overlap, and (iii) rescue-specific reward on peer-disagreement cases.


We propose \fullmethodname{} (\methodname) (Figure \ref{fig:framework1}), an algorithm for training multiple LLMs to collaborate on reasoning tasks. \methodname{} is built around a simple micro-round protocol: (1) in a \emph{cold} round, each model samples several independent solution traces; (2) if at least one model succeeds, we extract out a successful trace as hint; and (3) in a \emph{contexted} round, each model resamples with hint, and models that previously failed are eligible for an explicit \emph{rescue bonus} when they recover. The key intuition is that "peer success" provides a targeted supervision signal precisely on the instances where one model needs help.

Our goal is not only to improve single-model accuracy but also to reduce \emph{error overlap} between collaborators, thereby increasing the probability that \emph{at least one} model solves a problem. We operationalize this with an explore--exploit reward design. Exploitation rewards correctness (with optional partial credit), while exploration rewards encourage diverse reasoning traces using a greedy Determinantal Point Process (DPP) approximation. Our \textit{key} contributions are three-fold:
\begin{enumerate}
    \item We propose \fullmethodname{} (\methodname), a novel algorithm to jointly train $N$ LLM policies with a two-round protocol--a \emph{cold} round of independent sampling followed by a \emph{contexted rescue} round where failed models learn from hints extracted from successful peers, with \emph{hint dropout} to improve robustness.
    \item We propose a combined reward that simultaneously optimizes (i) correctness, (ii) exploration-exploitation of complementary traces (reducing correlated failures), and (iii) an explicit \emph{rescue bonus} to incentivize successful recovery on peer-disagreement.
    \item Extensive empirical analysis that show across GSM8K, MATH, AIME, GPQA, and MedMCQA, \methodname{} consistently improves collaborative performance, achieving \textbf{99.54\% Oracle Team Pass@2} on GSM8K and strong gains on harder benchmarks (e.g., \textbf{92.08\%} on MATH, \textbf{79.65\%} on AIME, and \textbf{77.34\%} on GPQA), while training on no more than 1{,}000 examples.
\end{enumerate}



\section{Related Work}

\paragraph{Reasoning via prompting, search, and inference-time ensembling.}
Prompting methods elicit intermediate reasoning traces, most notably chain-of-thought \citep{wei2022chain}, and can be strengthened with decompositional prompting such as least-to-most \citep{zhou2023leasttomost}.
Sampling-based inference improvements like self-consistency \citep{wang2023selfconsistency} reduce variance by aggregating multiple traces, while search-based frameworks explicitly explore reasoning branches (e.g., tree-of-thought \citep{yao2023tree}).
Other lines interleave reasoning with actions or external tools (e.g., ReAct \citep{yao2023react}).
Model combination methods further improve accuracy by selecting or fusing candidate generations (e.g., ranking-and-fusion in LLM-Blender \citep{jiang2023llmblender}).
In contrast, \methodname{} moves collaboration \emph{into training}: peer successes become an on-policy learning signal (via hints + rescue), rather than only an inference-time aggregation trick.

\paragraph{Reinforcement learning and preference optimization for LLM post-training.}
Alignment and post-training commonly rely on RLHF-style objectives \citep{ouyang2022training} or preference optimization such as DPO \citep{rafailov2023direct}.
For reasoning-focused post-training, group-based policy optimization (e.g., GRPO) has shown strong empirical performance \citep{shao2024deepseekmath}, and recent work proposes more stable variants such as sequence-level clipping (GSPO) \citep{zheng2025gspo} and smooth, token-adaptive trust-region control (SAPO) \citep{gao2025sapo}. Prior work has also applied preference optimization directly to chain-of-thought traces to improve mathematical reasoning \citep{lahlou2025port}.
\methodname{} is largely orthogonal to these optimizer choices: we contribute a \emph{multi-model data collection protocol} at training time itself and reward shaping aimed at reducing correlated failures.

\paragraph{Diversity, exploration, and reducing correlated failures.}
Diversity is a classic lever for improving ensemble performance and reducing error correlation \citep{dietterich2000ensemble,kuncheva2003measures}.
In sequence generation, diverse decoding procedures such as Diverse Beam Search encourage exploration of multiple modes \citep{vijayakumar2016diversebeam}.
Determinantal point processes (DPPs) provide a principled repulsive model for selecting diverse subsets \citep{kulesza2012dpp}.
The recent work SD-E$^2$ \cite{mishra2026sde2semanticexplorationreasoning} explores diverse reasoning traces within a single small model.
Building on these ideas, we introduce a lightweight greedy approximation (\emph{DPP-lite}) over hybrid semantic/structural distances between reasoning traces, and use it as an explicit exploration reward to reduce \emph{overlap} between collaborators rather than merely diversifying outputs at inference time.


\begin{figure*}[t]
  \centering
  \includegraphics[width=\textwidth,height=0.5\textheight,keepaspectratio]{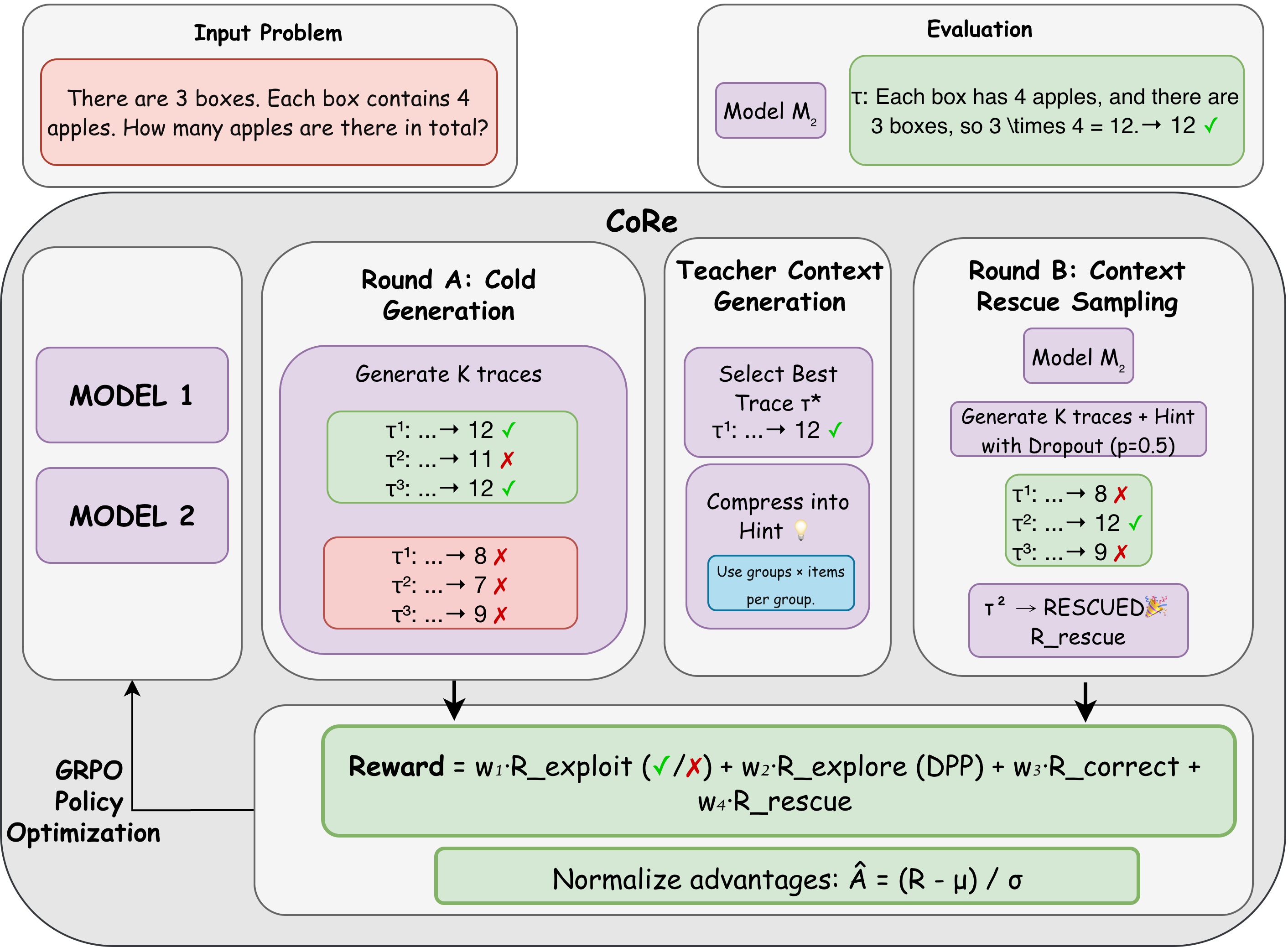}
  \caption{\textbf{\methodname{} training.} In \textsc{Round A}, each model samples $K$ independent traces. If any trace succeeds, the best one is used as teacher context for \textsc{Round B}, where models resample with explore--exploit rewards and a rescue bonus for successful recoveries.}
  \label{fig:framework2}
\end{figure*}

\paragraph{Multi-agent and multi-model collaboration.}
Multi-agent prompting and debate-style protocols can improve factuality and reasoning at inference time by having multiple instances interact \citep{irving2018ai,du2023multiagentdebate}.
Systems work has also proposed general multi-agent orchestration frameworks (e.g., CAMEL \citep{li2023camel} and AutoGen \citep{wu2023autogen}) that coordinate role-based collaboration, tool use, and iterative refinement.
Closer to our setting, MALT trains a sequential generator--verifier--refiner pipeline using trajectory expansion and value-style credit assignment \citep{motwani2025malt}, while MAPoRL co-trains multiple LLMs with multi-agent RL and verifier rewards over multi-turn discussions \citep{park2025maporl}.
\methodname{} differs in its training signal: peer success on the \emph{same instance} is compressed into a hint for failed collaborators, and the rescue reward explicitly optimizes recovery on disagreement cases.
Thus, \methodname{} is not an inference-time discussion protocol or a fixed role pipeline; it is an on-policy cross-teaching protocol that reshapes the collaborators' policies during post-training.

\paragraph{Collaborative learning, distillation, and self-training with rationales.}
Knowledge distillation and peer learning transfer information across models by matching outputs or training on teacher-generated data \citep{hinton2015distilling,zhang2018deep,furlanello2018born}.
For reasoning, prior work also leverages \emph{rationales} as supervision via bootstrapping loops (STaR) \citep{zelikman2022star} or distilling step-by-step solutions \citep{hsieh2023distilling}.
\methodname{} differs in two key ways: (i) peer supervision is generated \emph{on-policy} and \emph{instance-conditioned} (only when a peer succeeds on the same example), and (ii) learning is explicitly targeted to the \emph{disagreement set} via rescue bonuses.

\section{Problem Formulation}

Let $\mathcal{D}=\{(x,y^*)\}$ denote a dataset of reasoning problems.
We train a set of $N$ autoregressive policies $\{\pi_{\theta_i}\}_{i=1}^N$.
A \emph{reasoning trace} $\tau$ is a generated token sequence (including intermediate steps) and induces a final answer via an extraction function $\mathrm{Ans}(\tau)$.
We define correctness as $\mathbb{I}[\mathrm{Ans}(\tau)=y^*]$.

\paragraph{Team success.}
A cooperative objective is to maximize the probability that \emph{at least one} model solves a problem:
\begin{equation}
\label{eq:team_obj}
\max_{\theta_1,\dots,\theta_N}\ 
\mathbb{E}_{(x,y^*)\sim\mathcal{D}}
\Big[
\mathbb{I}\big[\exists i\ \mathrm{s.t.}\ \mathrm{Ans}(\tau_i)=y^*\big]
\Big],
\end{equation}
where $\tau_i\sim \pi_{\theta_i}(\cdot\mid x)$.
This objective is sparse and does not indicate how to credit individual traces or encourage complementary reasoning.
We therefore optimize a dense, per-trace surrogate reward. 

For each problem $x$, models first sample a \emph{cold} set of traces, then optionally resample in a \emph{context-conditioned} rescue round using a hint derived from a successful peer.
Rescue traces are explicitly rewarded, steering learning toward instances where one model can help another. This turns collaboration from an inference-time heuristic into a training-time objective aligned with \cref{eq:team_obj}.

\section{Method}

For each training problem $x$, \methodname{} collects traces in two micro-rounds.

\paragraph{\textsc{Round A} (cold).}
Each model $i$ samples $K$ traces independently:
$\{\tau^{A}_{i,k}\}_{k=1}^K \sim \pi_{\theta_i}(\cdot\mid x)$.
We compute rewards for all sampled traces and identify whether the model succeeded:
$s_i = \max_k \mathbb{I}[\mathrm{Ans}(\tau^{A}_{i,k})=y^*]$.

\paragraph{Teacher context construction.}
If any model succeeds in \textsc{Round A}, we select the highest-reward correct trace and build a teacher context $c$ with a length budget $L$ tokens.


\paragraph{\textsc{Round B} (contexted rescue).}
Each model samples $K'$ additional traces conditioned on the context with \emph{hint dropout}:
with probability $p_{\text{hint}}$ the model sees $c$, otherwise it sees the original prompt.
A \textsc{Round B} trace is marked as \emph{rescue-eligible} only when the model failed in \textsc{Round A} ($s_i=0$) and the hint is provided.

\subsection{Reward Design}

\paragraph{Exploitation reward.}
For a trace $\tau$ we compute
\begin{equation}
R_{\text{et}}(\tau)
=
\mathbb{I}[\mathrm{Ans}(\tau)=y^*]
+
\alpha\, \mathrm{Pl}(\mathrm{Ans}(\tau),y^*),
\end{equation}
where $\mathrm{Pl}(\cdot,\cdot)\in[0,1]$ is a lightweight overlap-based partial score and $\alpha\in[0,1]$ downweights partial credit (default $\alpha=0.3$).

\paragraph{Exploration reward (DPP-lite).}
Given the set of traces for a single problem, we construct a diverse subset $S$ with a greedy DPP approximation.
For any trace $\tau\notin S$, we define
\begin{equation}
R_{\text{ee}}(\tau)
=
\max\!\Big(
0,\ 
\min_{\tau'\in S} d(\tau,\tau') - \delta
\Big),
\end{equation}
where $d$ is a hybrid semantic/structural distance.
Traces in $S$ receive zero exploration reward.

\paragraph{Rescue bonus.}
If a model fails in \textsc{Round A} but produces a correct trace in \textsc{Round B} \emph{with the hint provided}, we add a fixed rescue bonus $r_{\text{teach}}$.

\paragraph{Combined reward.}
We define a binary rescue indicator $z(\tau)\in\{0,1\}$.
The combined reward is
\begin{equation}
R(\tau)
=
w_{1} R_{\text{et}}(\tau)
+
w_{2} R_{\text{ee}}(\tau)
+
r_{\text{teach}} z(\tau)
+
R_{\text{aux}}(\tau).
\end{equation}

\paragraph{Auxiliary rewards.}
We optionally include a trace-accuracy bonus and a cross-model complementarity term:
\begin{equation}
\label{eq:aux_reward}
\begin{aligned}
R_{\text{aux}}(\tau^{r}_{i,k})
&=
\lambda_{\text{ta}}\, \mathbb{I}[\mathrm{Ans}(\tau^{r}_{i,k})=y^*]\,(1-\rho_r(x)) \\
&\quad+
\lambda_{\text{cm}}\, R_{\text{cross}}(\tau^{r}_{i,k}),
\end{aligned}
\end{equation}
where $r\in\{A,B\}$, $K_A=K$, $K_B=K'$, and
$\rho_r(x)=\frac{1}{N K_r}\sum_{j,k'}\mathbb{I}[\mathrm{Ans}(\tau^{r}_{j,k'})=y^*]$
is the fraction of correct traces for problem $x$ in round $r$.
The factor $(1-\rho_r(x))$ upweights correct traces on harder or borderline instances where few traces succeed, acting as a soft curriculum around the disagreement boundary where cross-teaching is most useful.
The complementarity term is secondary and is quality-gated below to avoid rewarding arbitrary diversity from poor traces.

\paragraph{Optional cross-model complementarity reward.}
We define a quality gate
$q^{r}_{j}=\mathbb{I}\!\left[\max_{k'} R_{\text{et}}(\tau^{r}_{j,k'})\ge \eta\right]$.
For a trace $\tau^{r}_{i,k}$,
\begin{equation}
\label{eq:cross_model_reward}
\begin{aligned}
R_{\text{cross}}(\tau^{r}_{i,k})
&=
q^{r}_{j}
\cdot
\max\!\Big(
0,\
\min_{k'} d(\tau^{r}_{i,k},\tau^{r}_{j,k'})
-
\delta_{\text{cm}}
\Big).
\end{aligned}
\end{equation}
In the main experiments we enable this term with a small weight (late-stage default $\lambda_{\text{cm}}=0.1$) and ablate it in Appendix~\ref{app:ablations}; setting $\lambda_{\text{cm}}=0$ recovers the simpler exploit--explore--rescue objective.

\paragraph{Policy optimization objective.}
Following group-based RL for LLMs \citep{shao2024deepseekmath}, we normalize rewards across all traces sampled for the same problem instance:
\begin{equation}
A(\tau) = \frac{R(\tau) - \mu}{\sigma + \epsilon},
\quad
\mu,\sigma \ \text{computed over traces for } x.
\end{equation}
This normalization stabilizes optimization and ensures that relative trace quality, rather than absolute reward scale, drives learning.
Traces from \textsc{Round B} (contexted rescue) are included in the update with a configurable weight $\lambda_B$ (default $0.8$), allowing the policy to learn from rescue behavior while keeping the cold-round signal dominant and preventing over-reliance on hints. The \methodname{} training architecture is shown in Figure \ref{fig:framework2}. 
The collaborative component of \methodname{} is \emph{loss-agnostic}: collaboration enters exclusively through (i) micro-round and (ii) reward construction, and can be combined with any of the group based policy gradient optimization objectives.



\begin{algorithm}[t]
\caption{\methodname{}}
\label{alg:core}
\begin{algorithmic}[1]
\REQUIRE Models $M_1,M_2$ with policies $\pi_{\theta_1},\pi_{\theta_2}$; dataset $\mathcal{D}$
\FOR{epoch $e=1$ to $E$}
  \FOR{problem $(x,y^*)\in\mathcal{D}$}
    \STATE \textbf{\textsc{Round A}:} sample $K$ traces per model; compute rewards; mark successes $s_1,s_2$
    \IF{$s_1\lor s_2$}
      \STATE build teacher context $c$ from best successful trace
    \ELSE
      \STATE $c \leftarrow \emptyset$
    \ENDIF
    \STATE \textbf{\textsc{Round B}:} sample $K'$ traces per model with hint dropout; add rescue bonus if applicable
    \STATE compute group-normalized advantages and update $\theta_1,\theta_2$ with chosen policy loss
  \ENDFOR
\ENDFOR
\end{algorithmic}
\end{algorithm}

\subsection{Decomposition Theorem for Collaboration Gain}
\label{sec:theory}

Our evaluation emphasizes \emph{team success}---whether any collaborator solves a problem.
For two models, team success admits a simple decomposition that clarifies what our collaboration objective optimizes and why reducing overlap matters.

\begin{theorem}[Collaboration gain equals complementary correctness]
\label{thm:gain}
Let $C_1,C_2\in\{0,1\}$ indicate whether models 1 and 2 solve an instance, and let $T=\mathbb{I}[C_1\lor C_2]$ denote the team-success indicator.
Define accuracies $p_i=\mathbb{P}(C_i=1)$ and joint success $p_{12}=\mathbb{P}(C_1=1,C_2=1)$.
Then
\begin{equation}
\label{eq:team_success_ie}
\mathbb{P}(T=1)=p_1+p_2-p_{12}.
\end{equation}
Assume without loss of generality that $p_2\ge p_1$.
Then the collaboration gain over the best single model,
$\Delta \triangleq \mathbb{P}(T=1)-p_2$, equals
\begin{equation}
\label{eq:gain_complement}
\Delta = \mathbb{P}(C_1=1, C_2=0).
\end{equation}
\end{theorem}

\begin{proof}
\Cref{eq:team_success_ie} follows from inclusion--exclusion:
\begin{equation}
\mathbb{P}(C_1\lor C_2)
=
\mathbb{P}(C_1)+\mathbb{P}(C_2)-\mathbb{P}(C_1\land C_2).
\end{equation}
For \cref{eq:gain_complement},
\begin{equation}
\begin{aligned}
\Delta
&= (p_1 + p_2 - p_{12}) - p_2 \\
&= p_1 - p_{12} \\
&= \mathbb{P}(C_1 = 1) \\
&\quad - \mathbb{P}(C_1 = 1, C_2 = 1) \\
&= \mathbb{P}(C_1 = 1, C_2 = 0).
\end{aligned}
\end{equation}
\end{proof}

\begin{remark}[Limits of collaboration gain]
\label{rem:gain_bounds}
Since $\Delta=\mathbb{P}(C_1=1,C_2=0)$, it is bounded as
\begin{equation}
\label{eq:gain_bounds}
0 \le \Delta \le \min\{p_1,\ 1-p_2\}.
\end{equation}
Thus, even perfect collaboration cannot exceed the weaker model's marginal accuracy $p_1$, nor can it exceed the stronger model's error rate $1-p_2$.
\end{remark}

\paragraph{Implications for training.}
\Cref{thm:gain} implies that beating the best single model requires increasing \emph{complementary correctness}--the weaker model must be right where the stronger one is wrong (i.e., lower success overlap $p_{12}$ for fixed $p_1$). \methodname{} targets this by promoting diverse traces via DPP-lite exploration and focusing learning on disagreement cases through the rescue signal.

\section{Experiments}
We study two collaborating small language model pairs, constraining the combined parameter count to at most $\sim$7B, to test whether \methodname{}'s collaboration gains persist beyond weak baselines and extend to stronger reasoning-tuned models.
Specifically, we consider: (1) a \emph{Qwen instruct} pair, \texttt{Qwen2.5-3B-Instruct}~\cite{qwen2024} and \texttt{Qwen3-4B-Instruct}~\cite{qwen3tr}, chosen to enable direct comparison with the SD-E$^2$ state-of-the-art baseline~\cite{mishra2026sde2semanticexplorationreasoning}; and (2) a \emph{compact reasoning} pair, \texttt{Phi-4-mini-reasoning}~\cite{xu2025phi4minireasoningexploringlimitssmall} and \texttt{Ministral-3-3B-Reasoning}~\cite{liu2026ministral3}. 
We report \emph{training-time} collaboration with \emph{cold} test-time evaluation (no inference-time hints); all models use parameter-efficient LoRA~\citep{hu2022lora}. Unless otherwise stated, team metrics are oracle Team Pass@$K$ and are used to measure error overlap; we also include non-oracle majority-vote and agreement-gated selection results for the 3-model AIME setting.

Each training step comprises two micro-rounds: \textsc{\textsc{Round A}} (cold), where each model samples $K$ independent traces from the original prompt, and \textsc{\textsc{Round B}} (contexted), where each model samples $K'$ traces with max\_token budget of 3072 while optionally conditioning on a teacher context extracted from the best successful peer trace in \textsc{Round A}. For training, we use batch\_size=4. To improve robustness, we apply hint dropout with probability $p_{\text{hint}}$, so \textsc{Round B} interleaves contexted and non-contexted rollouts. The teacher context is the peer trace's \emph{correct strategy block} (i.e., the strategy giving outcome same as the ground truth) inserted \emph{verbatim} as context; we cap it to a fixed budget (default 1536 tokens, via a lightweight estimate) and truncate any excess. We remove explicit answer lines/tags from the teacher context to prevent answer leakage; truncation-related failure modes are analyzed in Section~\ref{sec:error_analysis} and Appendix~\ref{app:config}.

All experiments optimize the combined reward. Crucially, we \textbf{enable} the cross-model complementarity term $R_{\text{cross}}$ with a small non-zero weight (default late-stage $w_{\text{cross}}=0.1$), and we ablate it in Appendix~\ref{app:ablations}.
We use $K=2$, $K'=1$, $p_{\text{hint}}=0.75$ and include the rescue bonus for successful \textsc{Round B} recovery. Our exploration and cross-model rewards depend on a hybrid distance $d(\tau,\tau')$ between traces.
We implement:
\begin{equation}
\begin{aligned}
d(\tau,\tau') \;=\;& \lambda_{\text{emb}}\Bigl(1-\cos\bigl(e(\tau),e(\tau')\bigr)\Bigr) \\
&+ \lambda_{\text{struct}}\Bigl(1-\mathrm{Jaccard}\bigl(op(\tau),op(\tau')\bigr)\Bigr).
\end{aligned}
\end{equation}

where $e(\tau)$ is a sentence embedding of the trace and $op(\tau)$ is a set of extracted operation signatures.
We use the Sentence-Transformers embedding model \texttt{sentence-transformers/all-MiniLM-L6-v2} \cite{reimers2019sentencebert,reimers2020allminilm, wang2020minilm} for $e(\cdot)$, and extract operations via simple regex triggers for algebraic/arithmetic and reasoning operators (e.g., addition/multiplication, simplification, substitution, case analysis).
Default weights are $\lambda_{\text{emb}}=0.6$ and $\lambda_{\text{struct}}=0.4$. Full extraction rules are provided in Appendix~\ref{app:config}.

\subsection{Datasets and Baselines}
\label{sec:benchmarks}

\paragraph{Datasets.}
We evaluate on four reasoning benchmarks spanning arithmetic, formal mathematics, contest-style problem solving, and scientific QA: GSM8K~\citep{cobbe2021gsm8k}, MATH~\citep{hendrycks2021math}, AIME, and GPQA~\citep{rein2023gpqa}.
Training follows a \emph{low-data regime}: each run trains on at most 1{,}000 examples (and fewer when the dataset is smaller).
For GSM8K, we use the standard train/test split and report results on the official test set.
For MATH, AIME, and GPQA, we construct an 80:20 train--test split and report results on the resulting test sets (MATH: 2{,}500; AIME: 193; GPQA main: 90).

\paragraph{Baselines.}
\label{sec:baselines}

We compare \methodname{} against three classes of baselines:

\begin{enumerate}
    \item \textbf{Base (off-the-shelf):} We evaluate the pretrained models without \methodname{} training. To separate collaboration gains from merely using two models, we also report a \emph{base-pair oracle}: run both base models independently under the same Pass@$K$ budget and count an instance correct if either model succeeds.


    \item \textbf{SD-E$^2$:} We compare against SD-E$^2$~\citep{mishra2026sde2semanticexplorationreasoning}, a strong diverse-exploration baseline for small-model reasoning. For the Qwen setting, we match their setup using \texttt{Qwen2.5-3B-Instruct} and \texttt{Qwen3-4B-Instruct}. We also report \emph{SD-E$^2$ + Oracle}, which combines the two SD-E$^2$ models (trained independently) and counts an instance correct if either model succeeds.

    \item \textbf{SD-E$^2$ at 7B--8B scale:} To contextualize scale under our $\sim$7B total-parameter budget, Appendix~\ref{app:additional_results} reports SD-E$^2$ results for \texttt{7B}--\texttt{8B} single models: \texttt{Qwen-2.5-7B-Instruct}~\cite{qwen2024}, \texttt{Phi-Small-8K-Instruct}~\cite{abdin2024phi3technicalreporthighly}, and \texttt{Ministral-3-8B-Reasoning}~\cite{liu2026ministral3}.

\end{enumerate}

\subsection{Evaluation metrics}
\label{sec:evaluation}
We perform all evaluation within a token budget of 4096 max\_tokens. We evaluate both \emph{individual} reasoning performance and the \emph{collaborative} benefits of training-time cross-teaching. For each model, Pass@$K$ measures whether at least one of $K$ sampled solution traces is correct.
In particular, Pass@1 uses a single trace (greedy or single-sample decoding), while Pass@2 samples two traces per model under stochastic decoding and counts an instance as solved if either trace is correct.
We report \textbf{Pass@1} and \textbf{Pass@2} to capture both deterministic performance and robustness under limited sampling. \textbf{Team Pass@$K$} extends Pass@$K$ to collaborators: an instance is counted as solved if \emph{any} model produces a correct solution among its $K$ traces.
This is an \emph{oracle} metric (it assumes perfect selection among candidates) and therefore should be read as an upper bound and a direct diagnostic of error overlap.
To estimate deployable performance without oracle access, we also evaluate two simple non-oracle selectors in the 3-model AIME setting: \textbf{majority vote}, which returns the most frequent extracted answer, and \textbf{agreement-gated selection}, which returns the unanimous answer when all models agree and otherwise falls back to majority vote.



\begin{table*}[t]
\caption{\textbf{Qwen pair results across datasets (Pass@1/Pass@2).} We report Base, SD-E$^2$~\citep{mishra2026sde2semanticexplorationreasoning}, and \methodname{} for Qwen2.5-3B-Instruct and Qwen3-4B-Instruct.}
\label{tab:qwen_table_all}
\centering
\small
\setlength{\tabcolsep}{3.2pt}
\begin{tabular}{lcccccccc}
\toprule
\textbf{Model} &
\multicolumn{2}{c}{\textbf{GSM8K}} &
\multicolumn{2}{c}{\textbf{MATH}} &
\multicolumn{2}{c}{\textbf{AIME}} &
\multicolumn{2}{c}{\textbf{GPQA}} \\
\cmidrule(lr){2-3} \cmidrule(lr){4-5} \cmidrule(lr){6-7} \cmidrule(lr){8-9}
& Pass@1 & Pass@2 & Pass@1 & Pass@2 & Pass@1 & Pass@2 & Pass@1 & Pass@2 \\
\midrule
\multicolumn{9}{c}{Base} \\
\midrule
Qwen2.5-3B-Instruct              & 54.66 & 56.30 & 32.40 & 38.10 & 6.74  & 7.50  & 4.20  & 5.90 \\
Qwen3-4B-Instruct                & 55.32 & 57.10 & 33.55 & 39.40 & 7.04  & 7.80  & 4.55  & 6.25 \\
Qwen2.5-3B + Qwen3-4B + Oracle   & 57.10 & 60.05 & 35.80 & 41.95 & 7.90  & 9.20  & 5.35  & 7.10 \\
\midrule
\multicolumn{9}{c}{SD-E$^2$} \\
\midrule
Qwen2.5-3B-Instruct              & 82.03 & 84.55 & 56.40 & 59.85 & 13.28 & 14.85 & 8.10  & 9.70 \\
Qwen3-4B-Instruct                & 82.50 & 85.05 & 57.10 & 60.40 & 13.30 & 14.90 & 8.55  & 10.10 \\
Qwen2.5-3B + Qwen3-4B + Oracle   & 85.02 & 87.20 & 59.30 & 63.10 & 17.14 & 17.90 & 10.20 & 11.95 \\
\midrule
\multicolumn{9}{c}{\textbf{\methodname{}}} \\
\midrule
Qwen2.5-3B-Instruct              & 91.62 & 92.95 & 70.56 & 71.96 & 15.70 & 16.82 & 11.90 & 13.05 \\
Qwen3-4B-Instruct                & 92.49 & 93.33 & 71.65 & 72.52 & 16.10 & 17.04 & 10.75 & 12.10 \\
Qwen2.5-3B + Qwen3-4B (Team)     & \textbf{99.10} & \textbf{99.54} & \textbf{79.40} & \textbf{80.80} & \textbf{19.50} & \textbf{20.45} & \textbf{16.25} & \textbf{17.80} \\
\bottomrule
\end{tabular}
\end{table*}

\begin{table*}[t]
\caption{\textbf{Pass@1/Pass@2 across datasets (GSM8K/MATH/AIME/GPQA).} Values are Pass@1 (single-sample) and Pass@2 (two samples).}
\label{tab:table2}
\centering
\small
\setlength{\tabcolsep}{3.2pt}
\begin{tabular}{lcccccccc}
\toprule
\textbf{Model} &
\multicolumn{2}{c}{\textbf{GSM8K}} &
\multicolumn{2}{c}{\textbf{MATH}} &
\multicolumn{2}{c}{\textbf{AIME}} &
\multicolumn{2}{c}{\textbf{GPQA}} \\
\cmidrule(lr){2-3} \cmidrule(lr){4-5} \cmidrule(lr){6-7} \cmidrule(lr){8-9}
& Pass@1 & Pass@2 & Pass@1 & Pass@2 & Pass@1 & Pass@2 & Pass@1 & Pass@2 \\
\midrule
\multicolumn{9}{c}{Base} \\
\midrule
Phi-4-mini-reasoning              & 72.40 & 78.10 & 57.04 & 68.16 & 11.91 & 16.58 & 10.35 & 12.23 \\
Ministral-3-3B-Reasoning             & 70.95 & 76.70 & 54.21 & 64.44 & 10.18 & 14.86 & 9.14  & 11.20 \\
Phi-4-mini + Ministral-3-3B + Oracle & 74.85 & 80.20 & 59.25 & 70.70 & 13.02 & 17.92 & 11.86 & 14.03 \\
\midrule
\multicolumn{9}{c}{SD-E$^2$} \\
\midrule
Phi-4-mini-reasoning              & 80.60 & 84.35 & 69.68 & 73.82 & 53.92 & 57.58 & 46.35 & 50.23 \\
Ministral-3-3B-Reasoning             & 79.10 & 83.40 & 67.25 & 73.15 & 49.00 & 55.42 & 42.13 & 44.38 \\
Phi-4-mini + Ministral-3-3B + Oracle & 83.25 & 85.50 & 71.85 & 77.70 & 54.42 & 60.19 & 48.12 & 55.26 \\
\midrule
\multicolumn{9}{c}{\textbf{\methodname{}}} \\
\midrule
Phi-4-mini-reasoning              & 90.14 & 91.10 & 85.24 & 90.01 & 68.16 & 76.64 & 67.76 & 75.26 \\
Ministral-3-3B-Reasoning             & 89.25 & 90.23 & 83.72 & 89.12 & 67.15 & 75.90 & 65.58 & 74.82 \\
\midrule
Phi-4-mini + Ministral-3-3B          & \textbf{95.72} & \textbf{96.72} & \textbf{88.42} & \textbf{92.08} & \textbf{71.42} & \textbf{79.65} & \textbf{70.18} & \textbf{77.34} \\
\bottomrule
\end{tabular}
\end{table*}

\section{Results and Analysis}
\label{sec:results}

Table~\ref{tab:qwen_table_all}, Table~\ref{tab:table2} summarizes results for the \texttt{Qwen-2.5-3B-instruct}, \texttt{Qwen-3-4B-instruct}, \texttt{Phi-4-mini-reasoning}~\cite{xu2025phi4minireasoningexploringlimitssmall} and \texttt{Ministral-3-3B-Reasoning}~\cite{liu2026ministral3} collaboration pair under three settings: Base (off-the-shelf), SD-E$^2$ training, and \methodname{} training.
Across all four benchmarks, \methodname{} yields consistent improvements in both \emph{individual} performance (Pass@1/Pass@2) and \emph{team} success, indicating gains beyond simply using two models at test time.

In Table~\ref{tab:qwen_table_all} it can be observed that
On GSM8K, \methodname{} lifts the two models from $\sim$55\% Base Pass@1 to 91.62\% and 92.49\% Pass@1, and from $\sim$56--57\% Base Pass@2 to 92.95\% and 93.33\% Pass@2 (Table~\ref{tab:qwen_table_all}). Using inclusion--exclusion with the per-model Pass@2 values in Table~\ref{tab:qwen_table_all}, the estimated probability that \emph{both are correct} is
$p_{12} \approx 92.95 + 93.33 - 99.54 = 86.74\%$,
leaving $\approx 12.80\%$ mass where exactly one model succeeds.
This complementary-correctness region is precisely what enables near-ceiling team reliability without requiring either member to be near-perfect.

A similar pattern holds on MATH: Team Pass@2 is 80.80\%, compared to 71.96\% and 72.52\% for the individual models.
Equivalently, only 19.20\% of problems remain unsolved by \emph{both} collaborators under a two-sample budget, which is a substantially lower both-wrong mass than the SD-E$^2$ oracle (36.90\%) and the Base oracle (58.05\%).
On GPQA, Team Pass@2 improves from 11.95\% (SD-E$^2$ oracle) and 7.10\% (Base oracle) to 17.80\% (\methodname{} team), indicating that cross-teaching still yields meaningful complementarity even in knowledge-heavy settings.

Although \methodname{} improves both individual and team performance on AIME/GPQA, these domains remain far from saturated for the Qwen pair, as also reflected in the comparison in \cref{fig:table1_summary}.
On AIME, Team Pass@2 improves from 9.20\% (Base oracle) and 17.90\% (SD-E$^2$ oracle) to 20.45\% (CoRe team), suggesting that many failures require longer-horizon insight than is consistently transferred through a truncated teacher context.

In Table~\ref{tab:table2} it can be observed that the pair of \texttt{Phi-4-mini-reasoning}~\cite{xu2025phi4minireasoningexploringlimitssmall} and \texttt{Ministral-3-3B-Reasoning}~\cite{liu2026ministral3} starts stronger than the Qwen models and hence benefit substantially from post-training. The Base oracle pair provides only modest gains over the stronger Base model (e.g., GSM8K Pass@2: 80.20\% vs.\ 78.10\%; MATH Pass@2: 70.70\% vs.\ 68.16\%), indicating that off-the-shelf diversity is limited.
SD-E$^2$ improves both competence and the oracle pair (e.g., MATH Pass@2 oracle: 77.70\%), but \methodname{} yields substantially larger gains (MATH Team Pass@2: 92.08\%). Thus, we argue that cross-teaching changes the learned policies in a way that increases the probability that \emph{at least one} model succeeds, rather than merely exploiting pre-existing diversity.

Across all datasets, \methodname{} improves Pass@2 for both members and the team (Table~\ref{tab:table2}).
For example, Team Pass@2 improves from 85.50\% (SD-E$^2$ oracle) to 96.72\% (CoRe team) on GSM8K, from 77.70\% to 92.08\% on MATH, and from 55.26\% to 77.34\% on GPQA, as also summarized visually in \cref{fig:table2_summary}.
These improvements indicate that \methodname{} increases the probability mass on correct solution modes under stochastic decoding, not only the best trace under deterministic decoding.

\subsection{Non-oracle selection}
\label{sec:non_oracle_selection}

Because Team Pass@$K$ is an oracle upper bound, we additionally evaluate deployable selection on a 3-model AIME team, where majority voting is meaningful. Table~\ref{tab:non_oracle_aime} shows that majority vote captures most of the oracle gain: the gap between oracle Team Pass@2 and majority vote is only 1.40 percentage points. Agreement-gated selection is slightly more conservative but remains above the best single model. These results suggest that \methodname{} not only increases the chance that a correct answer exists among candidates, but also makes agreement a more useful confidence signal by reducing correlated failure.

\begin{table}[t]
\caption{\textbf{Non-oracle selection on AIME.} The 3-model team uses Phi-4-mini-reasoning, Ministral-3-3B-Reasoning, and Ministral-3-8B-Reasoning.}
\label{tab:non_oracle_aime}
\centering
\small
\setlength{\tabcolsep}{4pt}
\begin{tabular}{lc}
\toprule
\textbf{Selection strategy} & \textbf{AIME Pass@2} \\
\midrule
Best single model & 78.45 \\
Agreement-gated team & 80.40 \\
Majority-vote team & 81.20 \\
Oracle Team Pass@2 & 82.60 \\
\bottomrule
\end{tabular}
\end{table}

\subsection{Collaboration Help via Cross Teaching}
\label{sec:results_mechanism_summary}

\methodname{} improves team success primarily by reducing \emph{correlated failures} and increasing \emph{complementary correctness} between collaborators under a fixed sampling budget. A direct diagnostic is the both-wrong probability $1-p_{\mathrm{team}}$ under a fixed Pass@$K$ protocol.
On the strong pair (Table~\ref{tab:table2}, Pass@2), Team accuracy implies that only 3.28\% (GSM8K), 7.92\% (MATH), 20.35\% (AIME), and 22.66\% (GPQA) of instances remain unsolved by \emph{both} models.
These reductions in correlated failure relative to SD-E$^2$ are the most direct driver of oracle team improvements. 

Training curves (Figure~\ref{fig:train_curves1}, ~\ref{fig:train_curves2}, ~\ref{fig:train_curves3}, ~\ref{fig:train_curves4}) show steady reward improvement on GSM8K, MATH, AIME and GPQA, with faster stabilization on arithmetic and slower, noisier progress on scientific QA.
This is consistent with cross-teaching concentrating learning signal on high-value disagreement cases where one model can provide a useful solution pattern or missing step for its partner. Within-instance exploration (DPP-lite) mitigates mode collapse among a model's sampled traces, while the cross-model complementarity term $R_{\text{cross}}$ explicitly discourages redundant high-quality traces across collaborators.
Together, these incentives reshape the solution distribution so that success mass is spread across different reasoning modes, explaining why \methodname{} can outperform compute-matched oracle baselines formed from independently trained pairs (Tables~\ref{tab:qwen_table_all} and~\ref{tab:table2}) rather than merely benefiting from naive ensembling.

\begin{figure}[t]
  \centering
  \includegraphics[width=\columnwidth]{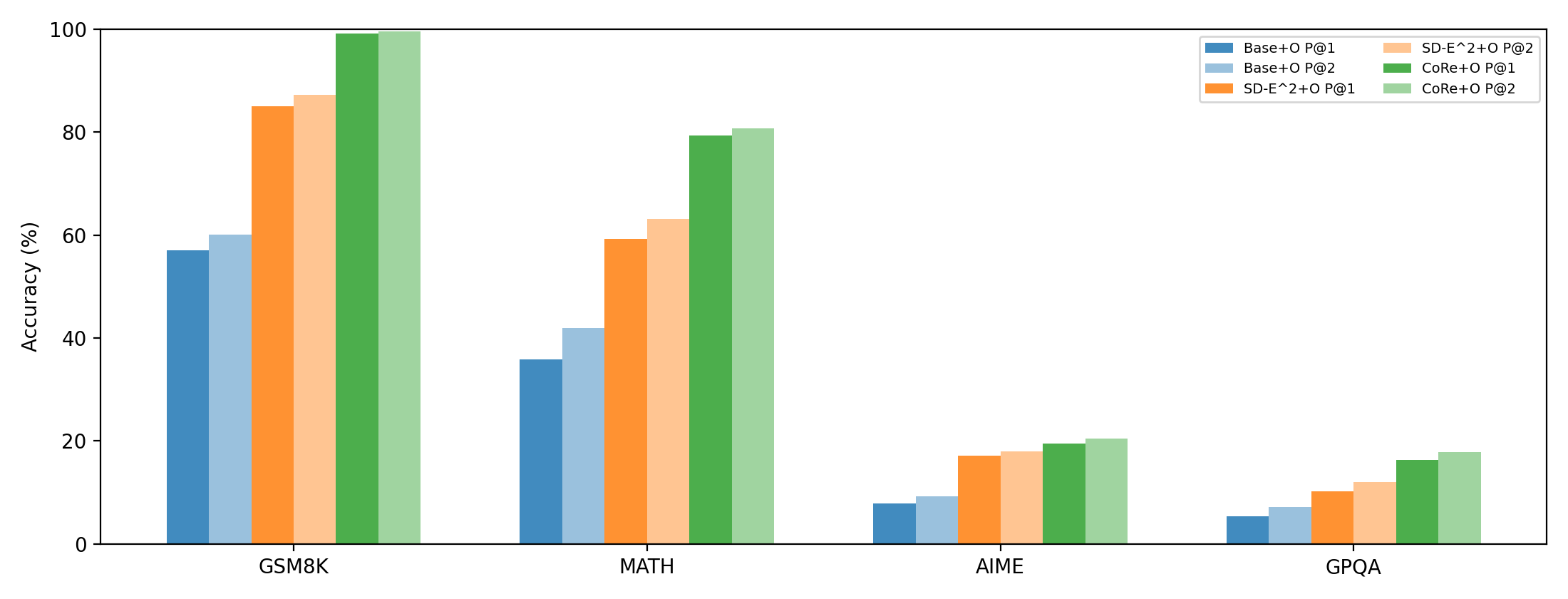}
  \caption{Base/SD-E$^2$ oracle baselines vs.\ the \methodname{} team for the Qwen2.5-3B + Qwen3-4B pair.}
  \label{fig:table1_summary}
\end{figure}


\begin{figure}[t]
  \centering
  \includegraphics[width=\columnwidth]{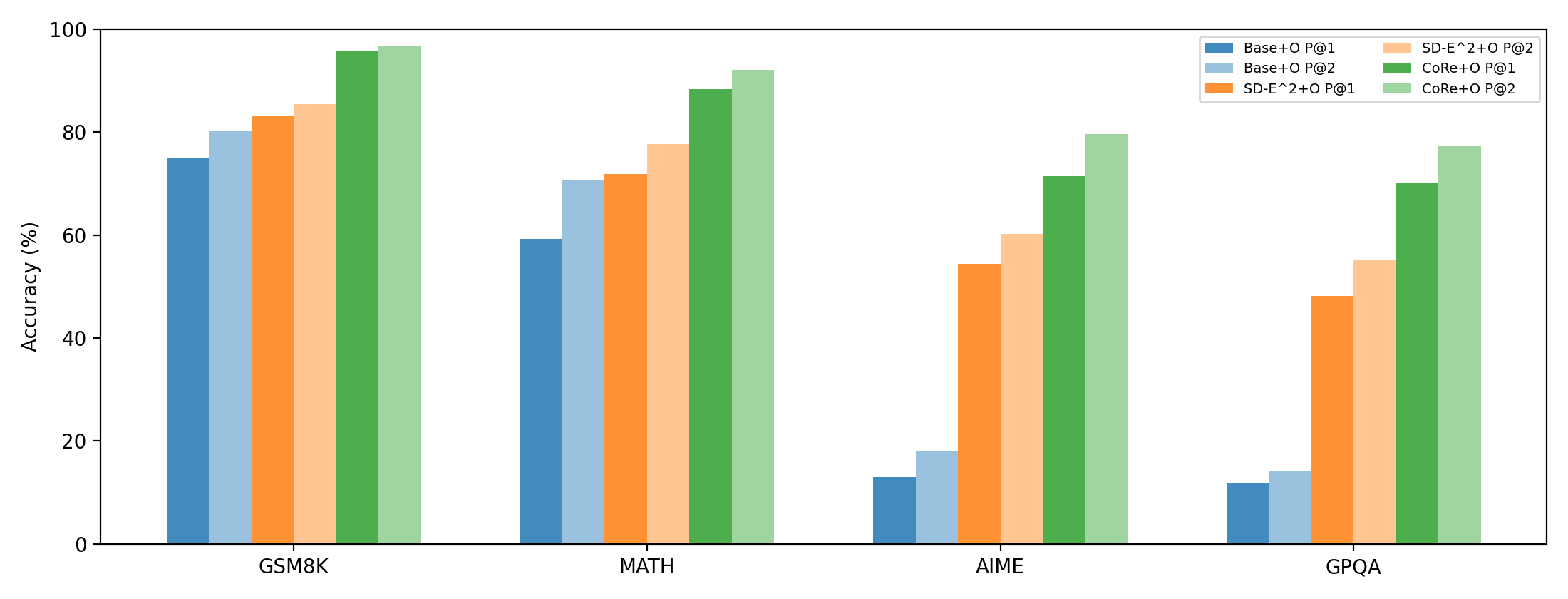}
  \caption{Base/SD-E$^2$ oracle baselines vs.\ the \methodname{} team for the Phi-4-mini + Ministral-3-3B pair.}
  \label{fig:table2_summary}
\end{figure}

\section{Error Analysis}
\label{sec:error_analysis}

We analyze remaining failures through two lenses:
(1) \textbf{task-intrinsic reasoning difficulty} (where both models fail regardless of collaboration),
and (2) \textbf{collaboration-specific failure modes} (where cross-teaching could help but fails due to hint selection, truncation, or brittleness).

A common failure pattern is \emph{early-step correctness followed by drift}: both models choose the right plan but make a sign error, lose a factor, or mis-handle fractions/percentages mid-way.
These errors are hard for a peer hint to repair if the hint does not explicitly highlight the fragile transformation.
AIME amplifies this issue: many problems require a single "unlock" insight plus careful downstream algebra, so partial correctness is not enough. Another shared failure is missing edge cases (e.g., boundary values, domain restrictions, parity constraints).
Even when models attempt case splits, they often prune too aggressively. This is particularly visible in contest-style problems where the entire solution hinges on enumerating all valid configurations. On GPQA, failures are frequently driven by missing domain facts or overconfident elimination.
Two models often converge to the same plausible distractor because it matches superficial patterns in the question.

Because teacher context is capped to a fixed budget, long peer traces can be truncated.
If the truncated hint retains setup but drops the decisive step (e.g., a key identity or elimination rationale), the rescued model may confidently follow an incomplete plan.
This motivates (i) selecting shorter teacher traces when multiple correct traces exist, or (ii) learned selection/truncation strategies.

A peer trace can be correct yet poorly justified or idiosyncratic; using it as teacher context may propagate brittle shortcuts or induce partial copying that fails to generalize. We mitigate this by selecting the highest exploitation-reward correct trace, stripping explicit answers, and applying hint dropout, but this risk remains.
Our exploration and cross-model complementarity rewards further counteract such mode collapse by encouraging alternative, more robust reasoning paths.

Our hybrid distance is lightweight by design, but it can mis-score diversity:
two traces may appear "diverse" under embeddings while sharing the same core mistaken assumption, or appear "similar" despite differing at the critical step.
This suggests a future direction of verifier-informed distance or structured symbolic signatures for specific domains (e.g., algebraic transformations for MATH).

\section{Conclusion}
\label{sec:conclusion}

We proposed \fullmethodname{} (\methodname), a training-time collaboration framework that converts \emph{peer success} into an on-policy learning signal for reasoning LLMs.
\methodname{} couples a two-round micro-protocol (cold sampling followed by contexted resampling with teacher context) with an explore--exploit reward that simultaneously improves correctness and discourages redundant solution modes via within-instance diversity and cross-model complementarity.
Across GSM8K, MATH, AIME, and GPQA, \methodname{} improves both single-model Pass@$K$ and oracle Team Pass@$K$, indicating meaningful reductions in correlated failures and increased complementary coverage at a fixed small-model parameter budget.
Our extensive empirical analysis of our results support our proposed training-time collaboration as a practical route to more reliable reasoning without relying exclusively on scaling model size or introducing heavy inference-time search. 

\section*{Impact Statement}

\methodname{} may have positive impact by increasing the reliability of reasoning systems using multiple small, open models, which can improve robustness for educational tutoring, scientific assistance, and safety-critical decision support where single-model failures are costly.
By explicitly reducing correlated failures, collaboration-trained systems may also be easier to monitor and verify using downstream selection mechanisms (e.g., verifiers or rankers), potentially improving trustworthiness in deployment. Potential risks include misuse for automated cheating or large-scale generation of convincing but incorrect solutions, as well as increased computational and environmental cost from running multiple models.
Collaboration could also amplify shared biases if both collaborators inherit similar harmful behaviors from data or prompts, and oracle-style metrics may overstate real-world performance when selection is imperfect.
Mitigations include restricting deployment in assessment settings, pairing \methodname{} with robust verification and safety filtering, reporting non-oracle selection results alongside oracle Team Pass@$K$, and using cost-aware serving policies (e.g., invoking the second model only when uncertainty is high).

\bibliography{example_paper}

@article{cobbe2021gsm8k,
  title={Training Verifiers to Solve Math Word Problems},
  author={Cobbe, Karl and Kosaraju, Vineet and Bavarian, Mohammad and Chen, Mark and Jun, Heewoo and Kaiser, Lukasz and Plappert, Matthias and Tworek, Jerry and Hilton, Jacob and Nakano, Reiichiro and others},
  journal={arXiv preprint arXiv:2110.14168},
  year={2021}
}

@article{hendrycks2021math,
  title={Measuring Mathematical Problem Solving With the MATH Dataset},
  author={Hendrycks, Dan and Burns, Collin and Kadavath, Saurav and Arora, Akul and Basart, Steven and Tang, Eric and Song, Dawn and Steinhardt, Jacob},
  journal={NeurIPS},
  year={2021}
}

@article{wei2022chain,
  title={Chain-of-Thought Prompting Elicits Reasoning in Large Language Models},
  author={Wei, Jason and Wang, Xuezhi and Schuurmans, Dale and Bosma, Maarten and Ichter, Brian and Xia, Fei and Chi, Ed and Le, Quoc and Zhou, Denny},
  journal={NeurIPS},
  year={2022}
}

@article{wang2023selfconsistency,
  title={Self-Consistency Improves Chain of Thought Reasoning in Language Models},
  author={Wang, Xuezhi and Wei, Jason and Schuurmans, Dale and Le, Quoc and Chi, Ed and Narang, Sharan and Chowdhery, Aakanksha and Zhou, Denny},
  journal={ICLR},
  year={2023}
}

@article{ouyang2022training,
  title={Training Language Models to Follow Instructions with Human Feedback},
  author={Ouyang, Long and Wu, Jeffrey and Jiang, Xu and Almeida, Diogo and Wainwright, Carroll and Mishkin, Pamela and Zhang, Chong and Agarwal, Sandhini and Slama, Katarina and Ray, Alex and others},
  journal={NeurIPS},
  year={2022}
}

@article{rafailov2023direct,
  title={Direct Preference Optimization: Your Language Model is Secretly a Reward Model},
  author={Rafailov, Rafael and Sharma, Archit and Mitchell, Eric and Ermon, Stefano and Manning, Christopher D and Finn, Chelsea},
  journal={NeurIPS},
  year={2023}
}

@article{shao2024deepseekmath,
  title={DeepSeekMath: Pushing the Limits of Mathematical Reasoning in Open Language Models},
  author={Shao, Zhihong and Wang, Peiyi and Zhu, Qihao and Xu, Runxin and Song, Junxiao and Zhang, Mingchuan and Li, YK and Wu, Y and Guo, Daya},
  journal={arXiv preprint arXiv:2402.03300},
  year={2024}
}

@article{zheng2025gspo,
  title={Group Sequence Policy Optimization},
  author={Zheng, Chujie and Liu, Shixuan and Li, Mingze and Chen, Xiong-Hui and Yu, Bowen and Gao, Chang and Dang, Kai and Liu, Yuqiong and Men, Rui and Yang, An and Zhou, Jingren and Lin, Junyang},
  journal={arXiv preprint arXiv:2507.18071},
  year={2025}
}

@article{gao2025sapo,
  title={Soft Adaptive Policy Optimization},
  author={Gao, Chang and Zheng, Chujie and Chen, Xiong-Hui and Dang, Kai and Liu, Shixuan and Yu, Bowen and Yang, An and Bai, Shuai and Zhou, Jingren and Lin, Junyang},
  journal={arXiv preprint arXiv:2511.20347},
  year={2025}
}

@article{jiang2023llmblender,
  title={LLM-Blender: Ensembling Large Language Models with Pairwise Ranking and Generative Fusion},
  author={Jiang, Dongfu and Ren, Xiang and Lin, Bill Yuchen},
  journal={ACL},
  year={2023}
}

@article{irving2018ai,
  title={AI Safety via Debate},
  author={Irving, Geoffrey and Christiano, Paul and Amodei, Dario},
  journal={arXiv preprint arXiv:1805.00899},
  year={2018}
}

@article{hinton2015distilling,
  title={Distilling the Knowledge in a Neural Network},
  author={Hinton, Geoffrey and Vinyals, Oriol and Dean, Jeff},
  journal={arXiv preprint arXiv:1503.02531},
  year={2015}
}

@article{furlanello2018born,
  title={Born Again Neural Networks},
  author={Furlanello, Tommaso and Lipton, Zachary C and Tschannen, Michael and Itti, Laurent and Anandkumar, Anima},
  journal={ICML},
  year={2018}
}

@article{hu2022lora,
  title={LoRA: Low-Rank Adaptation of Large Language Models},
  author={Hu, Edward J and Shen, Yelong and Wallis, Phillip and Allen-Zhu, Zeyuan and Li, Yuanzhi and Wang, Shean and Wang, Lu and Chen, Weizhu},
  journal={ICLR},
  year={2022}
}

@article{qwen2024,
  title={Qwen2.5 Technical Report},
  author={Qwen Team},
  journal={arXiv preprint arXiv:2412.15115},
  year={2024}
}

@article{kulesza2012dpp,
  title={Determinantal Point Processes for Machine Learning},
  author={Kulesza, Alex and Taskar, Ben},
  journal={Foundations and Trends in Machine Learning},
  volume={5},
  number={2--3},
  pages={123--286},
  year={2012}
}

@article{qwen3tr,
  title={Qwen3 Technical Report},
  author={Yang, An and Li, Anfeng and Yang, Baosong and Zhang, Beichen and Hui, Binyuan and Zheng, Bo and Yu, Bowen and Gao, Chang and Huang, Chengen and Lv, Chenxu and Zheng, Chujie and Liu, Dayiheng and Zhou, Fan and Huang, Fei and Hu, Feng and Ge, Hao and Wei, Haoran and Lin, Huan and Tang, Jialong and Yang, Jian and Tu, Jianhong and Zhang, Jianwei and Yang, Jianxin and Yang, Jiaxi and Zhou, Jing and Zhou, Jingren and Lin, Junyang and Dang, Kai and Bao, Keqin and Yang, Kexin and Yu, Le and Deng, Lianghao and Li, Mei and Li, Mingze and Xue, Mingfeng and Zhang, Pei and Wang, Peng and Zhu, Qin and Men, Rui and Gao, Ruize and Liu, Shixuan and Luo, Shuang and Li, Tianhao and Tang, Tianyi and Yin, Wenbiao and Ren, Xingzhang and Wang, Xinyu and Zhang, Xinyu and Ren, Xuancheng and Fan, Yang and Su, Yang and Zhang, Yichang and Zhang, Yinger and Wan, Yu and Liu, Yuqiong and Wang, Zekun and Cui, Zeyu and Zhang, Zhenru and Zhou, Zhipeng and Qiu, Zihan},
  journal={arXiv preprint arXiv:2505.09388},
  year={2025}
}

@inproceedings{rein2023gpqa,
  title     = {{{GPQA}: A Graduate-Level Google-Proof Q\&A Benchmark}},
  author    = {Rein, David and Hou, Betty Li and Stickland, Asa Cooper and Petty, Jackson and Pang, Richard Yuanzhe and Dirani, Julien and Michael, Julian and Bowman, Samuel R.},
  booktitle = {First Conference on Language Modeling},
  year      = {2024},
  url       = {https://openreview.net/forum?id=Ti67584b98},
  note      = {arXiv:2311.12022}
}

@article{yao2023tree,
  title   = {Tree of thoughts: Deliberate problem solving with large language models},
  author  = {Yao, Shunyu and Yu, Dian and Zhao, Jeffrey and Shafran, Izhak and Griffiths, Tom and Cao, Yuan and Narasimhan, Karthik},
  journal = {Advances in neural information processing systems},
  volume  = {36},
  pages   = {11809-11822},
  year    = {2023}
}

@inproceedings{du2024debate,
  title     = {Improving Factuality and Reasoning in Language Models through Multiagent Debate},
  author    = {Du, Yilun and Li, Shuang and Torralba, Antonio and Tenenbaum, Joshua B. and Mordatch, Igor},
  booktitle = {Proceedings of the 41st International Conference on Machine Learning},
  series    = {Proceedings of Machine Learning Research},
  volume    = {235},
  pages     = {11733--11763},
  year      = {2024},
  publisher = {PMLR},
  url       = {https://proceedings.mlr.press/v235/du24e.html}
}

@inproceedings{dietterich2000ensemble,
  title={Ensemble methods in machine learning},
  author={Dietterich, Thomas G},
  booktitle={International workshop on multiple classifier systems},
  pages={1--15},
  year={2000},
  organization={Springer}
}

@article{kuncheva2003diversity,
  title   = {Measures of Diversity in Classifier Ensembles and Their Relationship with the Ensemble Accuracy},
  author  = {Kuncheva, Ludmila I. and Whitaker, Christopher J.},
  journal = {Machine Learning},
  volume  = {51},
  pages   = {181--207},
  year    = {2003},
  doi     = {10.1023/A:1022859003006}
}

@inproceedings{lightman2024verify,
  title     = {Let's Verify Step by Step},
  author    = {Lightman, Hunter and Kosaraju, Vineet and Burda, Yura and Edwards, Harri and Baker, Bowen and Lee, Teddy and Leike, Jan and Schulman, John and Sutskever, Ilya and Cobbe, Karl},
  booktitle = {International Conference on Learning Representations (ICLR)},
  year      = {2024},
  url       = {https://arxiv.org/abs/2305.20050}
}

@inproceedings{foerster2018coma,
  title     = {Counterfactual Multi-Agent Policy Gradients},
  author    = {Foerster, Jakob N. and Farquhar, Gregory and Afouras, Triantafyllos and Nardelli, Nantas and Whiteson, Shimon},
  booktitle = {Proceedings of the AAAI Conference on Artificial Intelligence},
  pages     = {2974--2982},
  year      = {2018},
  doi       = {10.1609/AAAI.V32I1.11794}
}

@inproceedings{zhang2018deep,
  title={Deep mutual learning},
  author={Zhang, Ying and Xiang, Tao and Hospedales, Timothy M and Lu, Huchuan},
  booktitle={Proceedings of the IEEE conference on computer vision and pattern recognition},
  pages={4320--4328},
  year={2018}
}

@article{yao2023react,
  title   = {ReAct: Synergizing Reasoning and Acting in Language Models},
  author  = {Yao, Shunyu and Zhao, Jeffrey and Yu, Dian and Du, Nan and Shafran, Izhak and Narasimhan, Karthik and Cao, Yuan},
  journal = {arXiv preprint arXiv:2210.03629},
  year    = {2023},
  url     = {https://arxiv.org/abs/2210.03629},
  note    = {ICLR 2023 camera-ready version on arXiv}
}

@article{zhou2023leasttomost,
  title   = {Least-to-Most Prompting Enables Complex Reasoning in Large Language Models},
  author  = {Zhou, Denny and Sch{\"a}rli, Nathanael and Hou, Le and Wei, Jason and Scales, Nathan and Wang, Xuezhi and Schuurmans, Dale and Cui, Claire and Bousquet, Olivier and Le, Quoc and Chi, Ed},
  journal = {arXiv preprint arXiv:2205.10625},
  year    = {2023},
  url     = {https://arxiv.org/abs/2205.10625},
  note    = {ICLR 2023 version on arXiv}
}

@article{kuncheva2003measures,
  title     = {Measures of Diversity in Classifier Ensembles and Their Relationship with the Ensemble Accuracy},
  author    = {Kuncheva, Ludmila I. and Whitaker, Christopher J.},
  journal   = {Machine Learning},
  volume    = {51},
  number    = {2},
  pages     = {181--207},
  year      = {2003},
  publisher = {Kluwer Academic Publishers},
  doi       = {10.1023/A:1022859003006},
  url       = {https://doi.org/10.1023/A:1022859003006}
}

@article{vijayakumar2016diversebeam,
  title   = {Diverse Beam Search: Decoding Diverse Solutions from Neural Sequence Models},
  author  = {Vijayakumar, Ashwin K. and Cogswell, Michael and Selvaraju, Ramprasath R. and Sun, Qing and Lee, Stefan and Crandall, David and Batra, Dhruv},
  journal = {arXiv preprint arXiv:1610.02424},
  year    = {2016},
  url     = {https://arxiv.org/abs/1610.02424}
}

@article{du2023multiagentdebate,
  title   = {Improving Factuality and Reasoning in Language Models through Multiagent Debate},
  author  = {Du, Yilun and Li, Shuang and Torralba, Antonio and Tenenbaum, Joshua B. and Mordatch, Igor},
  journal = {arXiv preprint arXiv:2305.14325},
  year    = {2023},
  url     = {https://arxiv.org/abs/2305.14325}
}

@inproceedings{lahlou2025port,
  title={Port: Preference optimization on reasoning traces},
  author={Lahlou, Salem and Abubaker, Abdalgader and Hacid, Hakim},
  booktitle={Proceedings of the 2025 Conference of the Nations of the Americas Chapter of the Association for Computational Linguistics: Human Language Technologies (Volume 1: Long Papers)},
  pages={10989--11005},
  year={2025}
}

@article{li2023camel,
  title   = {CAMEL: Communicative Agents for ``Mind'' Exploration of Large Language Model Society},
  author  = {Li, Guohao and Hammoud, Hasan Abed Al Kader and Itani, Hani and Khizbullin, Dmitrii and Ghanem, Bernard},
  journal = {arXiv preprint arXiv:2303.17760},
  year    = {2023},
  url     = {https://arxiv.org/abs/2303.17760},
  note    = {Accepted at NeurIPS 2023 (per arXiv)}
}

@article{wu2023autogen,
  title   = {AutoGen: Enabling Next-Gen LLM Applications via Multi-Agent Conversation},
  author  = {Wu, Qingyun and Bansal, Gagan and Zhang, Jieyu and Wu, Yiran and Li, Beibin and Zhu, Erkang and Jiang, Li and Zhang, Xiaoyun and Zhang, Shaokun and Liu, Jiale and Awadallah, Ahmed Hassan and White, Ryen W. and Burger, Doug and Wang, Chi},
  journal = {arXiv preprint arXiv:2308.08155},
  year    = {2023},
  url     = {https://arxiv.org/abs/2308.08155}
}

@article{zelikman2022star,
  title   = {STaR: Bootstrapping Reasoning With Reasoning},
  author  = {Zelikman, Eric and Wu, Yuhuai and Mu, Jesse and Goodman, Noah D.},
  journal = {arXiv preprint arXiv:2203.14465},
  year    = {2022},
  url     = {https://arxiv.org/abs/2203.14465},
  note    = {NeurIPS 2022 version exists; arXiv is the easiest stable citation}
}

@inproceedings{hsieh2023distilling,
  title     = {Distilling Step-by-Step! Outperforming Larger Language Models with Less Training Data and Smaller Model Sizes},
  author    = {Hsieh, Cheng-Yu and Li, Chun-Liang and Yeh, Chih-kuan and Nakhost, Hootan and Fujii, Yasuhisa and Ratner, Alex and Krishna, Ranjay and Lee, Chen-Yu and Pfister, Tomas},
  booktitle = {Findings of the Association for Computational Linguistics: ACL 2023},
  pages     = {8003--8017},
  year      = {2023},
  address   = {Toronto, Canada},
  publisher = {Association for Computational Linguistics},
  doi       = {10.18653/v1/2023.findings-acl.507},
  url       = {https://aclanthology.org/2023.findings-acl.507/}
}

@inproceedings{mishra2026sde2semanticexplorationreasoning,
    title = "{SD}-E2: Semantic Exploration for Reasoning Under Token Budgets",
    author = "Mishra, Kshitij  and
      Lukas, Nils  and
      Lahlou, Salem",
    editor = "Demberg, Vera  and
      Inui, Kentaro  and
      Marquez, Llu{\'i}s",
    booktitle = "Findings of the {A}ssociation for {C}omputational {L}inguistics: {EACL} 2026",
    month = mar,
    year = "2026",
    address = "Rabat, Morocco",
    publisher = "Association for Computational Linguistics",
    url = "https://aclanthology.org/2026.findings-eacl.323/",
    doi = "10.18653/v1/2026.findings-eacl.323",
    pages = "6144--6157",
    ISBN = "979-8-89176-386-9"
}

@misc{xu2025phi4minireasoningexploringlimitssmall,
      title={Phi-4-Mini-Reasoning: Exploring the Limits of Small Reasoning Language Models in Math}, 
      author={Haoran Xu and Baolin Peng and Hany Awadalla and Dongdong Chen and Yen-Chun Chen and Mei Gao and Young Jin Kim and Yunsheng Li and Liliang Ren and Yelong Shen and Shuohang Wang and Weijian Xu and Jianfeng Gao and Weizhu Chen},
      year={2025},
      eprint={2504.21233},
      archivePrefix={arXiv},
      primaryClass={cs.CL},
      url={https://arxiv.org/abs/2504.21233}, 
}

@misc{liu2026ministral3,
      title={Ministral 3}, 
      author={Alexander H. Liu and Kartik Khandelwal and Sandeep Subramanian and Victor Jouault and Abhinav Rastogi and Adrien Sadé and Alan Jeffares and Albert Jiang and Alexandre Cahill and Alexandre Gavaudan and Alexandre Sablayrolles and Amélie Héliou and Amos You and Andy Ehrenberg and Andy Lo and Anton Eliseev and Antonia Calvi and Avinash Sooriyarachchi and Baptiste Bout and Baptiste Rozière and Baudouin De Monicault and Clémence Lanfranchi and Corentin Barreau and Cyprien Courtot and Daniele Grattarola and Darius Dabert and Diego de las Casas and Elliot Chane-Sane and Faruk Ahmed and Gabrielle Berrada and Gaëtan Ecrepont and Gauthier Guinet and Georgii Novikov and Guillaume Kunsch and Guillaume Lample and Guillaume Martin and Gunshi Gupta and Jan Ludziejewski and Jason Rute and Joachim Studnia and Jonas Amar and Joséphine Delas and Josselin Somerville Roberts and Karmesh Yadav and Khyathi Chandu and Kush Jain and Laurence Aitchison and Laurent Fainsin and Léonard Blier and Lingxiao Zhao and Louis Martin and Lucile Saulnier and Luyu Gao and Maarten Buyl and Margaret Jennings and Marie Pellat and Mark Prins and Mathieu Poirée and Mathilde Guillaumin and Matthieu Dinot and Matthieu Futeral and Maxime Darrin and Maximilian Augustin and Mia Chiquier and Michel Schimpf and Nathan Grinsztajn and Neha Gupta and Nikhil Raghuraman and Olivier Bousquet and Olivier Duchenne and Patricia Wang and Patrick von Platen and Paul Jacob and Paul Wambergue and Paula Kurylowicz and Pavankumar Reddy Muddireddy and Philomène Chagniot and Pierre Stock and Pravesh Agrawal and Quentin Torroba and Romain Sauvestre and Roman Soletskyi and Rupert Menneer and Sagar Vaze and Samuel Barry and Sanchit Gandhi and Siddhant Waghjale and Siddharth Gandhi and Soham Ghosh and Srijan Mishra and Sumukh Aithal and Szymon Antoniak and Teven Le Scao and Théo Cachet and Theo Simon Sorg and Thibaut Lavril and Thiziri Nait Saada and Thomas Chabal and Thomas Foubert and Thomas Robert and Thomas Wang and Tim Lawson and Tom Bewley and Tom Bewley and Tom Edwards and Umar Jamil and Umberto Tomasini and Valeriia Nemychnikova and Van Phung and Vincent Maladière and Virgile Richard and Wassim Bouaziz and Wen-Ding Li and William Marshall and Xinghui Li and Xinyu Yang and Yassine El Ouahidi and Yihan Wang and Yunhao Tang and Zaccharie Ramzi},
      year={2026},
      eprint={2601.08584},
      archivePrefix={arXiv},
      primaryClass={cs.CL},
      url={https://arxiv.org/abs/2601.08584}, 
}

@misc{abdin2024phi3technicalreporthighly,
      title={Phi-3 Technical Report: A Highly Capable Language Model Locally on Your Phone}, 
      author={Marah Abdin and Jyoti Aneja and Hany Awadalla and Ahmed Awadallah and Ammar Ahmad Awan and Nguyen Bach and Amit Bahree and Arash Bakhtiari and Jianmin Bao and Harkirat Behl and Alon Benhaim and Misha Bilenko and Johan Bjorck and Sébastien Bubeck and Martin Cai and Qin Cai and Vishrav Chaudhary and Dong Chen and Dongdong Chen and Weizhu Chen and Yen-Chun Chen and Yi-Ling Chen and Hao Cheng and Parul Chopra and Xiyang Dai and Matthew Dixon and Ronen Eldan and Victor Fragoso and Jianfeng Gao and Mei Gao and Min Gao and Amit Garg and Allie Del Giorno and Abhishek Goswami and Suriya Gunasekar and Emman Haider and Junheng Hao and Russell J. Hewett and Wenxiang Hu and Jamie Huynh and Dan Iter and Sam Ade Jacobs and Mojan Javaheripi and Xin Jin and Nikos Karampatziakis and Piero Kauffmann and Mahoud Khademi and Dongwoo Kim and Young Jin Kim and Lev Kurilenko and James R. Lee and Yin Tat Lee and Yuanzhi Li and Yunsheng Li and Chen Liang and Lars Liden and Xihui Lin and Zeqi Lin and Ce Liu and Liyuan Liu and Mengchen Liu and Weishung Liu and Xiaodong Liu and Chong Luo and Piyush Madan and Ali Mahmoudzadeh and David Majercak and Matt Mazzola and Caio César Teodoro Mendes and Arindam Mitra and Hardik Modi and Anh Nguyen and Brandon Norick and Barun Patra and Daniel Perez-Becker and Thomas Portet and Reid Pryzant and Heyang Qin and Marko Radmilac and Liliang Ren and Gustavo de Rosa and Corby Rosset and Sambudha Roy and Olatunji Ruwase and Olli Saarikivi and Amin Saied and Adil Salim and Michael Santacroce and Shital Shah and Ning Shang and Hiteshi Sharma and Yelong Shen and Swadheen Shukla and Xia Song and Masahiro Tanaka and Andrea Tupini and Praneetha Vaddamanu and Chunyu Wang and Guanhua Wang and Lijuan Wang and Shuohang Wang and Xin Wang and Yu Wang and Rachel Ward and Wen Wen and Philipp Witte and Haiping Wu and Xiaoxia Wu and Michael Wyatt and Bin Xiao and Can Xu and Jiahang Xu and Weijian Xu and Jilong Xue and Sonali Yadav and Fan Yang and Jianwei Yang and Yifan Yang and Ziyi Yang and Donghan Yu and Lu Yuan and Chenruidong Zhang and Cyril Zhang and Jianwen Zhang and Li Lyna Zhang and Yi Zhang and Yue Zhang and Yunan Zhang and Xiren Zhou},
      year={2024},
      eprint={2404.14219},
      archivePrefix={arXiv},
      primaryClass={cs.CL},
      url={https://arxiv.org/abs/2404.14219}, 
}

@inproceedings{reimers2019sentencebert,
  title     = {Sentence-BERT: Sentence Embeddings using Siamese BERT-Networks},
  author    = {Reimers, Nils and Gurevych, Iryna},
  booktitle = {Proceedings of the 2019 Conference on Empirical Methods in Natural Language Processing and the 9th International Joint Conference on Natural Language Processing (EMNLP-IJCNLP)},
  year      = {2019},
  pages     = {3982--3992},
  publisher = {Association for Computational Linguistics},
  doi       = {10.18653/v1/D19-1410}
}

@article{wang2020minilm,
  title   = {MiniLM: Deep Self-Attention Distillation for Task-Agnostic Compression of Pre-Trained Transformers},
  author  = {Wang, Wenhui and Wei, Furu and Dong, Li and Bao, Hangbo and Yang, Nan and Zhou, Ming},
  journal = {arXiv preprint arXiv:2002.10957},
  year    = {2020},
  doi     = {10.48550/arXiv.2002.10957}
}

@misc{reimers2020allminilm,
  title        = {sentence-transformers/all-MiniLM-L6-v2},
  author       = {Reimers, Nils and {Sentence-Transformers} community},
  howpublished = {\url{https://huggingface.co/sentence-transformers/all-MiniLM-L6-v2}},
  note         = {Hugging Face model card (accessed 2026-01-29)},
  year         = {2020}
}

@misc{motwani2025malt,
  title         = {{MALT}: Improving Reasoning with Multi-Agent {LLM} Training},
  author        = {Motwani, Sumeet Ramesh and Smith, Chandler and Das, Rocktim Jyoti and Rafailov, Rafael and Laptev, Ivan and Torr, Philip H. S. and Pizzati, Fabio and Clark, Ronald and de Witt, Christian Schroeder},
  year          = {2025},
  eprint        = {2412.01928},
  archivePrefix = {arXiv},
  primaryClass  = {cs.LG},
  url           = {https://arxiv.org/abs/2412.01928},
  note          = {Published at COLM 2025}
}

@inproceedings{park2025maporl,
  title     = {{MAP}o{RL}: Multi-Agent Post-Co-Training for Collaborative Large Language Models with Reinforcement Learning},
  author    = {Park, Chanwoo and Han, Seungju and Guo, Xingzhi and Ozdaglar, Asuman E. and Zhang, Kaiqing and Kim, Joo-Kyung},
  booktitle = {Proceedings of the 63rd Annual Meeting of the Association for Computational Linguistics (Volume 1: Long Papers)},
  month     = jul,
  year      = {2025},
  address   = {Vienna, Austria},
  publisher = {Association for Computational Linguistics},
  pages     = {30215--30248},
  doi       = {10.18653/v1/2025.acl-long.1459},
  url       = {https://aclanthology.org/2025.acl-long.1459/}
}
\bibliographystyle{icml2026}

\appendix

\section{Reproducibility Details}

\subsection{Configuration summary}
\label{app:config}

Table~\ref{tab:hyperparams} lists key hyperparameters from the default configuration used in our main GSM8K runs.

\begin{table}[t]
\caption{Key hyperparameters.}
\label{tab:hyperparams}
\centering
\small
\begin{tabular}{lc}
\toprule
\textbf{Hyperparameter} & \textbf{Value} \\
\midrule
Cold traces ($K$) & 2 \\
Contexted traces ($K'$) & 1 \\
Hint dropout ($p_{\text{hint}}$) & 0.5 \\
Max teacher context tokens ($L$) & 1536 \\
Rescue bonus ($r_{\text{teach}}$) & 0.15 \\
Explore margin ($\delta$) & 0.15 \\
Explore set size cap & 10 \\
LoRA rank / alpha & 16 / 32 \\
LR (per model) & $10^{-5}$ \\
Max grad norm & 1.0 \\
\bottomrule
\end{tabular}
\end{table}

The cold prompt is:
\begin{quote}
\small
\texttt{Question: <x>\newline\newline Let's solve this step by step:}
\end{quote}

The contexted prompt injects the teacher context:
\begin{quote}
\small
\texttt{Question: <x>\newline\newline Hint:\newline <c>\newline\newline Let's solve this step by step:}
\end{quote}

We pass the teacher's reasoning \emph{directly} as context and enforce a fixed length cap of 1536 tokens. Given \textsc{Round A} samples across models, we select the \emph{highest exploitation-reward correct} trace as the teacher source. This reduces the probability of selecting spurious "lucky" correct outputs. We extract the \texttt{<strategy id="..."> ... </strategy>} block whose strategy-level outcome matches the final answer.
This yields a shorter, more focused teacher context than including all strategies. To prevent trivial leakage, we strip explicit answer markers (e.g., \texttt{Final answer:}, \texttt{\textbackslash boxed\{\}}, \texttt{<final\_answer>}) from the teacher context by default.
This encourages the rescued model to learn \emph{how} to solve the problem rather than copy the final token sequence. Truncation can remove critical steps for long proofs; we discuss this failure mode in Section~\ref{sec:error_analysis}.

Our evaluator extracts the final answer from common patterns such as "\#\#\#\# \texttt{<answer>}", "Final answer:", and bracketed formats. For numerical datasets like GSM8K, we normalize whitespace and punctuation before comparison.

\subsection{Compute cost and Efficiency}
\label{sec:compute_cost}

A key practical advantage of \methodname{} is that all training runs operate in a low-data regime (at most 1{,}000 training examples per dataset), whereas SD-E$^2$ fine-tunes on the full training split (7{,}473 instances).~\citep{mishra2026sde2semanticexplorationreasoning} %
Despite using substantially fewer training examples, \methodname{} attains large gains in both single-model Pass@$K$ and Team Pass@$K$ (Tables~\ref{tab:qwen_table_all} and~\ref{tab:table2}), indicating that cross-teaching provides a strong learning signal without requiring full-dataset RL runs.

Under our runs with batch size 4 and two traces per prompt, SD-E$^2$ training over the full $\sim$7.5k-example regime takes roughly 48--56 hours across backbones, while \methodname{} completes in roughly 20--24 hours when trained on 1{,}000 examples (same batch size and trace budget). This reflects a favorable trade-off: \methodname{} incurs additional per-step coordination (teacher selection, context construction, cross-model scoring), but the overall time-to-result is lower in our setting because training is intentionally sample-efficient.

\methodname{} introduces analogous overhead sources: (i) selecting a teacher trace and forming the capped teacher context, (ii) computing diversity/complementarity terms under the same hybrid distance used for DPP-lite and $R_{\text{cross}}$, and (iii) lightweight structural extraction for the trace-distance. In practice these costs are modest compared to decoding, but they are paid twice per step (micro-rounds A/B) and across both collaborators.

\paragraph{Complexity analysis.}
We provide a coarse-grained compute model per training step, separating (a) \emph{decoding}, which dominates wall-clock time, from (b) \emph{scoring/selection} overheads.
Let $N$ be the number of collaborating models (here $N=2$), $K$ be the number of cold traces in \textsc{Round A}, and $K'$ be the number of contexted traces in \textsc{Round B}.
Let $L$ denote the average generated trace length (tokens) and $L_c$ the teacher-context budget (tokens).
Let $C_{\text{gen}}(\cdot)$ be the per-token decoding cost for the backbone, and let $C_{\text{emb}}(\cdot)$ be the cost of computing a sentence embedding for a trace (used in $d(\tau,\tau')$).

\textbf{(1) Decoding cost.}
Each step samples $N(K+K')$ traces, so the dominant term scales as
\begin{equation}
\begin{aligned}
T_{\text{decode}}
&=\Theta\!\Big(N(K+K') \cdot L \cdot C_{\text{gen}}\Big) \\
&\quad + \Theta\!\Big(NK'\cdot L_c \cdot C_{\text{attn}}\Big).
\end{aligned}
\end{equation}
The second term captures the extra attention work from injecting the teacher context into \textsc{Round B}; in practice this is bounded by the fixed cap $L_c$ (default 1536) and amortized by the fact that $K'$ is small.

\textbf{(2) Teacher selection and context construction.}
Selecting the best successful teacher among $NK$ traces is linear:
\begin{equation}
T_{\text{select}}=\Theta(NK).
\end{equation}
Constructing the teacher context is a single pass over the chosen trace and therefore scales with its length:
\begin{equation}
T_{\text{context}}=\Theta(L)\quad\text{(with truncation to }L_c\text{)}.
\end{equation}
Both are negligible relative to decoding.

\textbf{(3) Distance-based rewards (DPP-lite and $R_{\text{cross}}$).}
Let $M=N(K+K')$ be the number of traces scored for the instance.
Computing embeddings for all traces is
\begin{equation}
T_{\text{emb}}=\Theta(M\cdot C_{\text{emb}}).
\end{equation}
Given embeddings and operation-signature sets, evaluating a single distance $d(\tau,\tau')$ is $O(1)$ for cosine distance plus $O(|op(\tau)|+|op(\tau')|)$ for Jaccard (small in practice due to coarse regex-triggered signatures).
The greedy DPP-lite construction with an explore-set cap $s_{\max}$ evaluates distances from each candidate to the current set; with capped set size,
\begin{equation}
T_{\text{DPP-lite}}=\Theta(M\cdot s_{\max}).
\end{equation}
The cross-model complementarity term compares each trace to its partner's traces (under a quality gate), yielding
\begin{equation}
T_{\text{cross}}=\Theta\!\big(NK_r^2\big)\ \text{ per round }r,
\end{equation}
which reduces to $\Theta(K^2 + K'^2)$ in the two-model setting and is small for our defaults ($K=2$, $K'=1$).

\textbf{(4) Total per-step complexity.}
Combining the above, the per-step runtime is
\begin{equation}
\begin{split}
T_{\text{step}}
=\Theta\!\Big(N(K+K')\cdot L\cdot C_{\text{gen}}\Big)
\;+\;\Theta\!\Big(M\cdot C_{\text{emb}}\\
\qquad\qquad\qquad\qquad\qquad\quad +\, M\cdot s_{\max} + NK + L\Big),
\end{split}
\end{equation}
where the first term (decoding) typically dominates, and the remaining terms are bounded by small constants under our capped budgets ($K,K',s_{\max},L_c$).
Thus, \methodname{} adds modest \emph{linear} overheads in the number of sampled traces, while the practical driver of wall-clock time remains the number/length of generated trajectories.

Overall, \methodname{} trades a small amount of extra coordination per training step for substantially improved \emph{team-level} reliability (reduced both-wrong mass) at a fixed small-model budget, and achieves competitive time-to-quality in a low-data regime. When compute or annotation budget is constrained, this "train less data, learn more from peers" behavior is a key part of \methodname{}'s efficiency.

\section{DPP-lite exploration reward properties}

\begin{proposition}[Separation of the diverse set]
\label{prop:separation}
Let $S$ be the diverse set returned by the greedy construction with margin $\delta$.
Then every pair of distinct traces $\tau,\tau'\in S$ satisfies $d(\tau,\tau')\ge \delta$.
\end{proposition}

\begin{proof}
A trace is added to $S$ only if its minimum distance to the current set is at least $\delta$.
This maintains the invariant that each newly added trace is at distance $\ge\delta$ from all existing members.
\end{proof}

\section{Policy optimization losses}
\label{app:losses}

For completeness, we summarize the loss functions implemented in \texttt{src/losses/}.
Let $r_t=\exp(\log\pi_\theta(a_t\mid s_t)-\log\pi_{\theta_{\text{old}}}(a_t\mid s_t))$ be the token-level importance ratio, and let $A$ denote a group-normalized advantage.

\paragraph{GRPO (token-level hard clipping).} \cite{shao2024deepseekmath}
GRPO uses PPO-style clipping at the token level:
\begin{equation}
\begin{aligned}
\mathcal{L}_{\mathrm{GRPO}}
&= -\mathbb{E}_t\Big[
\min\Big(
r_t A,\ 
\mathrm{clip}\!\big(r_t,1-\epsilon,1+\epsilon_{\text{high}}\big)A
\Big)
\Big] \\
&\quad + \beta\,\mathrm{KL}\!\big(\pi_\theta \,\Vert\, \pi_{\text{ref}}\big).
\end{aligned}
\end{equation}
\paragraph{GSPO (sequence-level ratio and clipping).} \cite{zheng2025gspo}
GSPO computes a sequence-level ratio using the geometric mean of token ratios:
\begin{equation}
 r_{\text{seq}}=\exp\Big(\tfrac{1}{|\tau|}\sum_t (\log\pi_\theta-\log\pi_{\theta_{\text{old}}})\Big),
\end{equation}
then applies tight sequence-level clipping before multiplying the mean log-probability of the sequence.

\paragraph{SAPO (soft sigmoid gate).}\cite{gao2025sapo}
SAPO replaces hard clipping with a smooth gate
$g_t=\sigma(\tau\,(r_t-1))\cdot (4/\tau)$,
with different temperatures for positive and negative advantages.
The loss uses gated advantages:
\begin{equation}
\mathcal{L}_{SAPO} = -\mathbb{E}_t\big[g_t\,A\,\log\pi_\theta(a_t\mid s_t)\big] + \beta\,\mathrm{KL}(\pi_\theta\Vert\pi_{\text{ref}}).
\end{equation}

\section{Additional benchmark results}
\label{app:additional_results}

We include a baseline comparison table for larger single models trained with SD-E$^2$ as well as an SD-E$^2$ oracle pair, evaluated on GSM8K, MATH, AIME, and GPQA.
This "collaboration vs.\ scaling/compute" question: \methodname{} constrains collaboration to a combined budget of at most $\sim$7B parameters, so we contextualize the gains against SD-E$^2$-trained 7B--8B single models and a compute-matched SD-E$^2$ two-model oracle.

\begin{table*}[t]
    \caption{\textbf{Additional baseline comparison (Pass@1).} SD-E$^2$ baselines include 7B--8B single models and an SD-E$^2$ oracle pair; \methodname{} uses a $\leq$7B two-model collaboration budget.}
    \label{tab:baseline_table}
    \centering
    \small
    \setlength{\tabcolsep}{5.5pt}
    \begin{tabular}{llcccc}
        \toprule
        \textbf{Model} & \textbf{Training} & \textbf{GSM8K} & \textbf{MATH} & \textbf{AIME} & \textbf{GPQA} \\
        \midrule
        Qwen2.5-7B-Instruct         & SD-E$^2$ & 86.21 & 44.00 & 26.00 & 30.00 \\
        Phi-3-small-8k-Instruct     & SD-E$^2$ & 78.00 & 34.00 & 20.00 & 26.00 \\
        Ministral-3-8B-Reasoning    & SD-E$^2$ & 84.50 & 47.00 & 51.50 & 33.00 \\
        \midrule
        Phi-4-mini + Ministral-3-3B + Oracle & SD-E$^2$ & 86.63 & 71.85 & 54.42 & 48.12 \\
        Phi-4-mini + Ministral-3-3B          & \textbf{\methodname{}} & \textbf{97.04} & \textbf{88.42} & \textbf{71.42} & \textbf{70.18} \\
        \bottomrule
    \end{tabular}
\end{table*}

The 7B--8B SD-E$^2$ single-model baselines provide a reference for capacity and within-model diverse-exploration training. The AIME values in Table~\ref{tab:baseline_table} are re-evaluated with the same 4096-token generation budget used in the main experiments to avoid truncating long contest-style solutions. Nevertheless, the \methodname{} team substantially exceeds these larger single models on all four benchmarks, despite operating under a smaller total parameter budget. This suggests that in the low-data post-training regime, many residual errors are not purely capacity-limited; they are often driven by systematic plan selection mistakes, brittle heuristics, or correlated blind spots that persist even when a model is encouraged to explore multiple traces.

SD-E$^2$ strengthens a \emph{single} policy by promoting trace diversity, but correlated failure can still dominate when the model repeatedly chooses the same wrong decomposition or elimination strategy.
In contrast, \methodname{} promotes \emph{cross-model} complementarity: when one collaborator finds a viable path, that success becomes a training signal for the partner, and complementarity rewards discourage both models from collapsing onto the same dominant reasoning mode. As a result, the probability that \emph{both} models fail decreases more aggressively than what is typically achieved by within-model diversity alone.

A strong compute-matched control is the SD-E$^2$ oracle pair (Phi-4-mini + Ministral-3-3B + Oracle), which already benefits from having two independently trained models and an oracle selector. That \methodname{} improves over this oracle across datasets indicates that gains cannot be explained by "just using two models" or by additional test-time compute. Instead, cross-teaching reshapes the learned policies: it improves individual competence while also shifting correct probability mass into complementary regions of the input space, thereby reducing overlap in errors and increasing oracle team success.

The largest gaps appear on the harder benchmarks (AIME/GPQA), where early wrong decisions often cascade and within-model sampling produces near-duplicate failures. Here, peer context provides a high-bandwidth corrective signal (e.g., a missing lemma, a successful elimination pattern, or a useful substitution), making collaboration training particularly effective.
On GSM8K, although the task is easier, \methodname{} still closes much of the remaining reliability gap by reducing residual correlated arithmetic and plan-selection failures.

It can be concluded through Table~\ref{tab:baseline_table} that \methodname{} delivers a qualitatively different improvement than modest parameter scaling or within-model diversity training: it explicitly targets correlated failures via training-time collaboration, yielding higher oracle reliability under tight data and compute budgets.

\section{Comparison with MALT and MAPoRL}
\label{app:malt_maporl}

Table~\ref{tab:malt_maporl_comparison} summarizes the closest collaborative training baselines raised during review. The comparison is not intended as a perfectly controlled leaderboard because model families, training data, and metrics differ; rather, it clarifies the algorithmic distinction. MALT uses a sequential role-specialized pipeline (generator, verifier, refiner) with trajectory expansion and value-style credit assignment, whereas MAPoRL uses multi-agent PPO with a verifier reward over multi-turn discussions. \methodname{} instead uses instance-conditioned cross-teaching: a correct peer trace on the same problem becomes a hint for failed collaborators, and the rescue reward directly trains recovery on disagreement cases.

\begin{table*}[t]
\caption{\textbf{Comparison to related collaborative/multi-agent post-training methods.} Reported numbers use each paper's own setting and are included to clarify evaluation context, not as a strictly controlled leaderboard.}
\label{tab:malt_maporl_comparison}
\centering
\small
\setlength{\tabcolsep}{3.0pt}
\begin{tabular}{p{0.15\textwidth}p{0.22\textwidth}p{0.30\textwidth}p{0.25\textwidth}}
\toprule
\textbf{Method} & \textbf{Collaboration stage} & \textbf{Learning signal} & \textbf{Representative setting} \\
\midrule
MALT~\citep{motwani2025malt} & Sequential multi-agent post-training & Search tree with value-style credit assignment & Llama-3.1-8B, MV@3 on GSM8K/MATH/CSQA \\
MAPoRL~\citep{park2025maporl} & Multi-agent RL co-training & Verifier reward over multi-turn discussion & Phi-3-mini scale, GSM8K/ANLI \\
\methodname{} & On-policy cross-teaching & Peer success on the same instance plus rescue reward & 3B--4B pairs, GSM8K/MATH/AIME/GPQA \\
\bottomrule
\end{tabular}
\end{table*}

The main methodological distinction is that \methodname{} creates a targeted correction signal only when a peer succeeds on the same instance. This differs from fixed role specialization, generic discussion rewards, and inference-time debate: the learning target is the disagreement set identified online during training. These methods are therefore complementary; for example, a verifier from MAPoRL or a MALT-style refiner could be used as a downstream selector for \methodname{} candidates.

\section{Hyperparameter and Distance Sensitivity}
\label{app:sensitivity}

Reviewers asked whether the reward design is fragile due to the number of components. Table~\ref{tab:hparam_sensitivity} sweeps the main task-specific parameters on MATH for the Phi-4-mini + Ministral-3B pair. Performance varies by at most about 0.9 percentage points across the tested ranges, suggesting that \methodname{} is not highly sensitive to exact tuning within reasonable ranges.

\begin{table}[t]
\caption{\textbf{Hyperparameter sensitivity on MATH.} Values are oracle Team Pass@2 for the Phi-4-mini + Ministral-3B pair.}
\label{tab:hparam_sensitivity}
\centering
\scriptsize
\setlength{\tabcolsep}{2.6pt}
\begin{tabular}{lcc}
\toprule
\textbf{Param.} & \textbf{Values tested} & \textbf{Team P@2} \\
\midrule
$w_1$ & $0.5/1.0/1.5$ & $91.22/92.08/91.85$ \\
$w_2$ & $0.1/0.2/0.4$ & $91.45/92.08/91.52$ \\
$r_{\text{teach}}$ & $0.25/0.50/0.75/1.0$ & $92.08/92.15/91.90/91.62$ \\
$p_{\text{hint}}$ & $0.5/0.75/1.0$ & $91.28$--$92.08$ \\
\bottomrule
\end{tabular}
\end{table}

We also ablate the distance metric used by DPP-lite and $R_{\text{cross}}$. Table~\ref{tab:distance_sensitivity} shows that replacing the hybrid semantic/structural distance with cosine-only distance produces only a small drop on AIME. Thus, the core gain comes from cross-teaching, while the hybrid distance provides a modest benefit by capturing operation-level differences that embeddings alone may miss.

\begin{table}[t]
\caption{\textbf{Distance metric sensitivity on AIME.} Values are oracle Team Pass@2 for the Phi-4-mini + Ministral-3B pair.}
\label{tab:distance_sensitivity}
\centering
\small
\begin{tabular}{lc}
\toprule
\textbf{Distance metric} & \textbf{AIME Team P@2} \\
\midrule
Hybrid semantic/structural (default) & 79.65 \\
Cosine-only & 78.70 \\
\bottomrule
\end{tabular}
\end{table}

\section{Scaling Beyond Two Collaborators}
\label{app:scaling}

The micro-round protocol extends naturally to $N>2$: \textsc{Round A} selects the best successful trace across all models, and \textsc{Round B} conditions failed models on the selected teacher context. Table~\ref{tab:three_model_scaling} reports a preliminary 3-model AIME experiment. Adding Ministral-3-8B-Reasoning improves oracle Team Pass@2 by 2.95 percentage points over the 2-model team, suggesting that additional heterogeneous collaborators can contribute complementary correctness when their error sets are not redundant.

\begin{table}[t]
\caption{\textbf{Preliminary scaling to three models on AIME.}}
\label{tab:three_model_scaling}
\centering
\small
\begin{tabular}{lc}
\toprule
\textbf{Configuration} & \textbf{Oracle Team Pass@2} \\
\midrule
2-model \methodname{} & 79.65 \\
3-model \methodname{} & 82.60 \\
\bottomrule
\end{tabular}
\end{table}

\section{Cross-dataset Generalization}
\label{app:transfer}

To test whether \methodname{} learns reusable reasoning strategies rather than only dataset-specific shortcuts, we evaluate AIME-trained \methodname{} models on MATH and GSM8K without additional fine-tuning. Table~\ref{tab:cross_dataset_transfer} shows favorable transfer: the AIME-trained team outperforms the base pair and the SD-E$^2$ oracle baseline, though per-dataset \methodname{} remains strongest.

\begin{table}[t]
\caption{\textbf{Zero-shot transfer from AIME-trained \methodname{}.} Values are Pass@1.}
\label{tab:cross_dataset_transfer}
\centering
\small
\begin{tabular}{lcc}
\toprule
\textbf{Setting} & \textbf{MATH} & \textbf{GSM8K} \\
\midrule
Base pair oracle & 59.25 & 74.85 \\
SD-E$^2$ oracle (per-dataset) & 71.85 & 83.25 \\
\methodname{} (AIME-trained $\rightarrow$ transfer) & 75.90 & 86.35 \\
\methodname{} Team (per-dataset) & 88.42 & 95.72 \\
\bottomrule
\end{tabular}
\end{table}

\section{AIME Split and Contamination Considerations}
\label{app:aime_contamination}

We use the AIME dataset from 1983--2025 and construct an 80:20 split (770 train / 193 test). Because many AIME problems are publicly available, contamination risk cannot be ruled out for any modern LLM. Our main conclusions rely on relative improvements under the same pretrained models and evaluation split: Base, SD-E$^2$, and \methodname{} are evaluated identically, so memorization risk affects all rows. We nevertheless report model-release cutoff considerations in Table~\ref{tab:aime_contamination} and treat absolute AIME accuracy with appropriate caution.

\begin{table*}[t]
\caption{\textbf{AIME contamination considerations.} Cutoffs are approximate and intended only to contextualize absolute scores; relative comparisons use the same models and test split.}
\label{tab:aime_contamination}
\centering
\small
\setlength{\tabcolsep}{4pt}
\begin{tabular}{lccc}
\toprule
\textbf{Model} & \textbf{Estimated training cutoff} & \textbf{AIME 2024 in data?} & \textbf{AIME 2025 in data?} \\
\midrule
Qwen2.5-3B-Instruct & late 2023 & unlikely & no \\
Qwen3-4B-Instruct & around Mar. 2025 & likely & possibly \\
Phi-4-mini-reasoning & around Jun. 2024 & likely & no \\
Ministral-3-3B-Reasoning & late 2024 & likely & unlikely \\
\bottomrule
\end{tabular}
\end{table*}

\section{Ablations and Design Choices}
\label{app:ablations}

We ablate the main components of \methodname{} to isolate which mechanisms drive gains.
To keep ablations comparable to the main results, we report \textbf{oracle Team Pass@$K$} under the same Pass@1/Pass@2 protocol used in Tables~\ref{tab:qwen_table_all} and~\ref{tab:table2}.
Unless otherwise noted, we use the default configuration ($K{=}2, K'{=}1$, hint dropout, DPP-lite, rescue bonus, and $R_{\text{cross}}$ enabled).

Table~\ref{tab:ablations_qwen_gsm8k} reports GSM8K team performance for the \texttt{Qwen2.5-3B-Instruct} + \texttt{Qwen3-4B-Instruct} pair, while Table~\ref{tab:ablations_table2} presents corresponding ablations for the stronger reasoning pair (\texttt{Phi-4-mini-reasoning} + \texttt{Ministral-3-3B-Reasoning}) on the harder benchmarks MATH, AIME, and GPQA.

\begin{table}[t]
    \caption{\textbf{Ablations on GSM8K for the Qwen team (oracle Team Pass@1/Pass@2).}}
    \label{tab:ablations_qwen_gsm8k}
    \centering
    \small
    \setlength{\tabcolsep}{4.0pt}
    \begin{tabular}{lcccc}
        \toprule
        \textbf{Variant} & \textbf{P@1} & $\Delta$ & \textbf{P@2} & $\Delta$ \\
        \midrule
        \methodname{} Team (Full) & 99.10 & -- & 99.54 & -- \\
        \quad w/o \textsc{Round B} (no rescue) & 98.20 & -0.90 & 98.90 & -0.64 \\
        \quad w/o Explore Reward & 98.70 & -0.40 & 99.20 & -0.34 \\
        \quad w/o Cross-Model Reward & 98.50 & -0.60 & 99.05 & -0.49 \\
        \midrule
        $K{=}2, K'{=}1$ (default) & 99.10 & -- & 99.54 & -- \\
        $K{=}4, K'{=}2$ & 99.20 & +0.10 & 99.60 & +0.06 \\
        \bottomrule
    \end{tabular}
\end{table}

\begin{table*}[t]
    \caption{\textbf{Ablations on MATH/AIME/GPQA.} Values are Pass@1/Pass@2 for Phi-4-mini + Ministral-3B.}
    \label{tab:ablations_table2}
    \centering
    \small
    \setlength{\tabcolsep}{3.2pt}
    \begin{tabular}{lcccccc}
        \toprule
        \textbf{Variant} &
        \multicolumn{2}{c}{\textbf{MATH}} &
        \multicolumn{2}{c}{\textbf{AIME}} &
        \multicolumn{2}{c}{\textbf{GPQA}} \\
        \cmidrule(lr){2-3} \cmidrule(lr){4-5} \cmidrule(lr){6-7}
        & P@1 & P@2 & P@1 & P@2 & P@1 & P@2 \\
        \midrule
        \methodname{} Team (Full) & 88.42 & 92.08 & 71.42 & 79.65 & 70.18 & 77.34 \\
        \quad w/o \textsc{Round B} (no rescue) & 84.50 & 88.90 & 66.20 & 74.80 & 64.80 & 72.20 \\
        \quad w/o Explore Reward & 86.70 & 90.60 & 69.40 & 78.10 & 68.00 & 75.90 \\
        \quad w/o Cross-Model Reward & 85.90 & 89.90 & 68.60 & 77.40 & 66.90 & 74.60 \\
        \quad Single Model (M1 only) & 85.24 & 90.01 & 68.16 & 76.64 & 67.76 & 75.26 \\
        \bottomrule
    \end{tabular}
\end{table*}

Across both the easier GSM8K regime (Table~\ref{tab:ablations_qwen_gsm8k}) and the harder benchmarks (Table~\ref{tab:ablations_table2}), removing \textsc{Round B} yields the largest degradation.
This supports the core mechanism of \methodname{}: when one model succeeds in \textsc{Round A}, its trace becomes a targeted correction signal for its partner on exactly the disagreement set.
On the strong pair, this effect is especially pronounced (e.g., MATH P@2 drops from 92.08 to 88.90; GPQA P@2 drops from 77.34 to 72.20), reflecting that many remaining errors are \emph{recoverable} given a correct peer plan or elimination pattern.

On GSM8K, removing $R_{\text{cross}}$ yields a modest team drop because performance is already near-saturated and the both-wrong mass is tiny under Pass@2.
However, on MATH/AIME/GPQA, ablating $R_{\text{cross}}$ consistently harms both Pass@1 and Pass@2 (Table~\ref{tab:ablations_table2}), indicating that explicit inter-model repulsion helps preserve complementary high-quality modes when the search space is larger and correlated failures are more common. Intuitively, DPP-lite prevents \emph{within-model} trace collapse, while $R_{\text{cross}}$ prevents \emph{between-model} redundancy--the latter becomes more important as tasks demand multiple distinct solution templates.

Ablating DPP-lite ("w/o Explore Reward") hurts performance consistently, but the magnitude is smaller than removing rescue on hard datasets. This suggests that diversity is most valuable as a guardrail against mode collapse rather than as the primary source of gains: it increases the probability that at least one trace explores a different reasoning path, which in turn enables more frequent successful teaching events. Increasing $K,K'$ provides only marginal improvements on GSM8K (Table~\ref{tab:ablations_qwen_gsm8k}), implying that once policies are strong, simply sampling more traces yields limited additional benefit. In contrast, the larger improvements from \textsc{Round B} and complementarity rewards indicate that the key bottleneck is not just candidate generation, but reshaping policies to reduce correlated failures and improve recovery dynamics.

Finally, Table~\ref{tab:ablations_table2} shows that the single-model baseline (M1 only) underperforms the full team, even though it shares the same backbone and training recipe.
This supports the central claim: \methodname{} is not merely regularizing a single policy, but actively reallocating correctness mass across collaborators by teaching on disagreement instances and discouraging redundant high-quality reasoning modes.

\begin{figure*}[t]
  \centering
  \includegraphics[width=\textwidth]{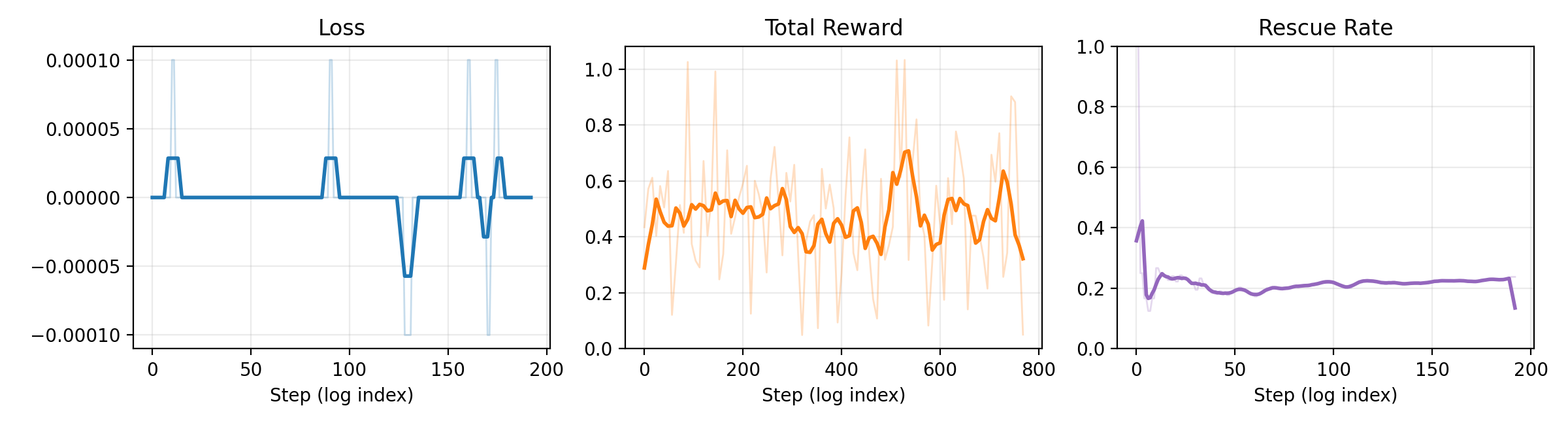}
  \caption{Train curves for AIME with \texttt{Phi-4-mini-reasoning} and \texttt{Ministral-3-3B-Reasoning}}
  \label{fig:train_curves1}
\end{figure*}

\begin{figure*}[t]
  \centering
  \includegraphics[width=\textwidth]{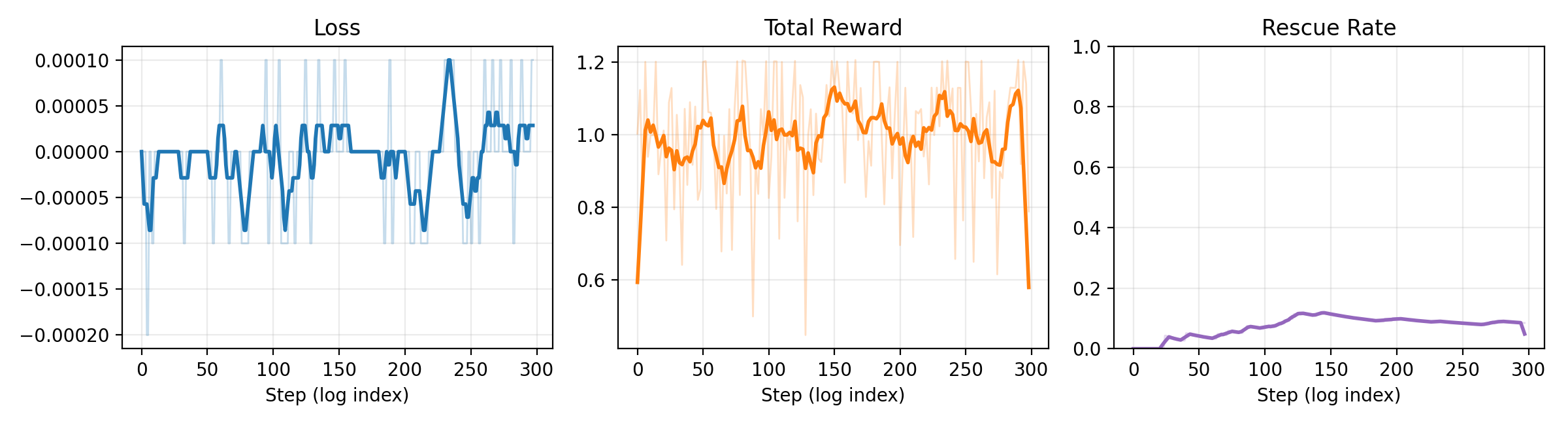}
  \caption{Train curves for GSM8K with \texttt{Phi-4-mini-reasoning} and \texttt{Ministral-3-3B-Reasoning}}
  \label{fig:train_curves2}
\end{figure*}
\begin{figure*}[t]
  \centering
  \includegraphics[width=\textwidth]{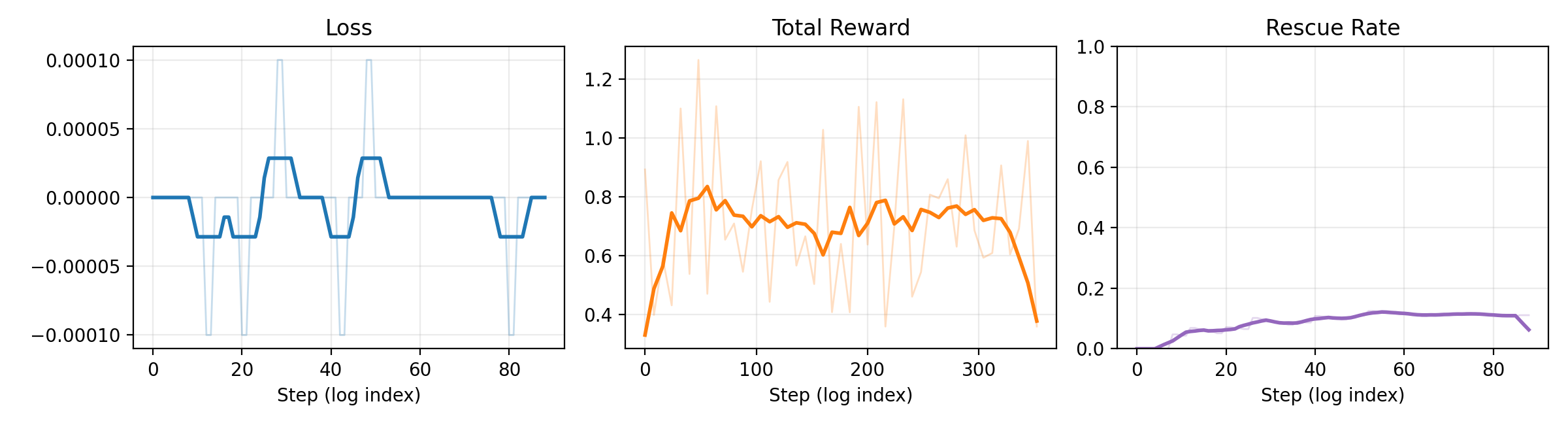}
  \caption{Train curves for GPQA with \texttt{Phi-4-mini-reasoning} and \texttt{Ministral-3-3B-Reasoning}}
  \label{fig:train_curves3}
\end{figure*}

\begin{figure*}[t]
  \centering
  \includegraphics[width=\textwidth]{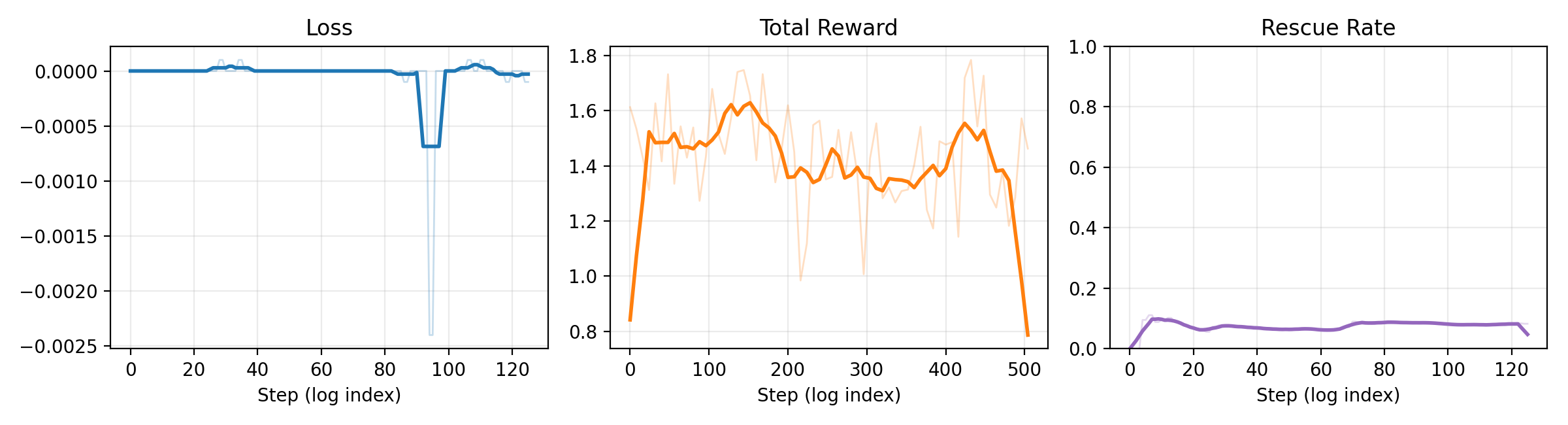}
  \caption{Train curves for MATH with \texttt{Phi-4-mini-reasoning} and \texttt{Ministral-3-3B-Reasoning}}
  \label{fig:train_curves4}
\end{figure*}

\section{Discussion}
\label{sec:discussion}

This section interprets the quantitative gains in Tables~\ref{tab:qwen_table_all} and~\ref{tab:table2}, the mechanism diagnostics in Section~\ref{sec:results_mechanism_summary}, and the component contributions in the ablations (Tables~\ref{tab:ablations_qwen_gsm8k} and~\ref{tab:ablations_table2}). We close with a qualitative analysis grounded in Table~\ref{tab:qualitative} and practical implications for deployment.

Team Pass@$K$ is an \emph{oracle} reliability measure: it reports whether \emph{any} of the $K$ samples from either collaborator is correct. Although it assumes an ideal selector, it is the cleanest diagnostic for \textbf{error overlap} and the \textbf{upper bound of selection}. In particular, Team Pass@$K$ directly reflects the probability of \emph{correlated failure} (the "both-wrong" mass), which is the limiting factor for multi-model systems under small sampling budgets.

The results show that \methodname{} increases Team Pass@$K$ for both pairs while also raising individual Pass@$K$ (Tables~\ref{tab:qwen_table_all},~\ref{tab:table2}).
This matters because oracle improvements can arise from two different sources:
(i) \emph{pure competence} (each model becomes more accurate), or
(ii) \emph{de-correlation} (models make different errors).
\methodname{} targets both: \textsc{Round B} converts a peer's success into an instance-conditioned correction signal, and $R_{\text{cross}}$ discourages redundant high-quality modes across collaborators.

A compute-matched baseline for collaboration is an oracle over two independently trained models (Tables~\ref{tab:qwen_table_all},~\ref{tab:table2}, Base/SD-E$^2$ oracle rows). These baselines already benefit from two-shot ensembling and diversity from different initializations/backbones.
However, \methodname{} still yields substantially larger Team Pass@$K$ gains, indicating that the improvement is not merely a test-time selection effect. The key difference is that \methodname{} couples the collaborators during training.
When one model succeeds in \textsc{Round A}, its successful trace provides a high-bandwidth hint for the partner in \textsc{Round B} on the \emph{same instance}. This converts sparse correctness signals into dense, targeted learning on exactly the disagreement set. Ablations confirm this: removing \textsc{Round B} causes the largest degradation for both the saturated GSM8K Qwen team (Table~\ref{tab:ablations_qwen_gsm8k}) and the stronger pair on harder datasets (Table~\ref{tab:ablations_table2}).

Within-model exploration (DPP-lite) combats mode collapse among a model's own samples, but does not explicitly prevent two models from converging to the \emph{same} strong-but-brittle strategy.
The cross-model term $R_{\text{cross}}$ addresses this by discouraging redundant high-quality traces across collaborators.
This distinction explains why $R_{\text{cross}}$ is modestly helpful on near-ceiling GSM8K but more consistently valuable on MATH/AIME/GPQA (Table~\ref{tab:ablations_table2}), where correlated failure dominates oracle limits.

The Qwen pair begins weaker but reaches near-ceiling team performance on GSM8K under \methodname{} (Table~\ref{tab:qwen_table_all}).
In contrast, on AIME/GPQA the Qwen team remains far from saturated, suggesting that some failures are driven by missing longer-horizon insights, factual knowledge gaps, or brittle transformations that are not reliably transferred by hints. The Phi+Ministral pair begins substantially stronger (Table~\ref{tab:table2}) and benefits more sharply on the harder benchmarks:
Team Pass@2 rises to 92.08 on MATH and 77.34 on GPQA, indicating that cross-teaching is especially effective when (i) each model has enough base competence to occasionally solve the instance, and (ii) remaining errors are often \emph{recoverable} given a correct peer decomposition or elimination rationale.
This aligns with the "teacher-available" region: \methodname{} can only rescue when at least one collaborator succeeds in \textsc{Round A}; as models become more competent, this region expands, creating a positive feedback loop in training.

\begin{table*}[t]
    \caption{\textbf{Qualitative examples from pass@1 evaluation logs.} M1 is Phi-4-mini; M2 is Ministral-3B. Entries show final answers extracted from the \texttt{<final\_answer>} tag when present. Row colors: blue = M2 correct, green = M1 correct, orange = both wrong.}
    \label{tab:qualitative}
    \centering
    \small
    \setlength{\tabcolsep}{3pt}
    \begin{tabular}{l p{0.42\textwidth} c c c p{0.16\textwidth}}
        \toprule
        \textbf{Dataset} & \textbf{Problem (abbrev.)} & \textbf{GT} & \textbf{M1} & \textbf{M2} & \textbf{Notes} \\
        \midrule
        \rowcolor{blue!6}
        GSM8K & Jill earns \$20/hr teaching and \$30/hr coaching, 35+15 hrs/week for 50 weeks; annual salary? & \texttt{57500} & \texttt{575} & \texttt{57500} & M2 correct; M1 truncates zeros. \\
        \rowcolor{blue!6}
        AIME & Two cubics share a complex root $m+\sqrt{n}\,i$; given real roots $-20$ and $-21$, find $m+n$. & \texttt{330} & \texttt{310} & \texttt{330} & M2 correct; M1 underestimates $m+n$. \\
        \rowcolor{blue!6}
        AIME & Isosceles right triangle with interior point $P$ s.t.\ $\angle PAB=\angle PBC=\angle PCA$ and $AP=10$; find area. & \texttt{250} & \texttt{200} & \texttt{250} & M2 correct; M1 misses geometric constraint. \\
        \rowcolor{green!7}
        MATH & Let $r$ be a root of $x^2+5x+7=0$; compute $(r-1)(r+2)(r+6)(r+3)$. & \texttt{13} & \texttt{13} & \texttt{001} & M1 correct; M2 collapses to small value. \\
        \rowcolor{blue!6}
        MATH & Integer roots of $2x^4+4x^3-5x^2+2x-3=0$ (comma-separated). & \texttt{1,-3} & \texttt{ANSWER} & \texttt{1, -3} & M2 correct; M1 formatting failure. \\
        \rowcolor{blue!6}
        GPQA & Enthalpy of neutralization for mixed HCl/H$_2$SO$_4$/Ba(OH)$_2$ (MCQ). & \texttt{A} & \texttt{B} & \texttt{A} & M2 correct; M1 picks distractor. \\
        \rowcolor{orange!8}
        GPQA & NaNH$_2$ on 1-bromobenzene-2-d: number of organic products (MCQ). & \texttt{A} & \texttt{C} & \texttt{B} & Both wrong; likely benzyne-mechanism confusion. \\
        \bottomrule
    \end{tabular}
\end{table*}

\subsection{Deployment Implications}
\label{sec:discussion_selection}

A practical system must select a single answer without oracle access.
Our results suggest that \methodname{} is a particularly strong fit for settings where selection is imperfect but improving:
\begin{itemize}
    \item \textbf{Verifier-based selection.} Since \methodname{} increases the probability that at least one candidate is correct, verifier capacity is spent distinguishing a smaller set of plausible solutions rather than searching for correctness from scratch. This is precisely the regime where lightweight verifiers tend to be most effective.
    \item \textbf{Disagreement-triggered routing.} Because collaboration gain concentrates on disagreement instances, a simple router can invoke the second model or additional sampling only when the first model's answer is uncertain (low log-prob margin) or when two collaborators disagree. This targets compute to the hard tail while keeping average serving cost low.
    \item \textbf{Agreement as a confidence heuristic.} Although agreement does not guarantee correctness, \methodname{} reduces redundant failure modes, so agreement becomes more informative: if both models converge, it is more likely to reflect a robust shared reasoning mode rather than a shared shortcut.
\end{itemize}
Importantly, these strategies are complementary to \methodname{}: better selection converts oracle gains into deployable gains, and \methodname{} makes selection easier by increasing complementary correctness.

\subsection{Qualitative Analysis}
\label{sec:qualitative}

We inspected pass@1 evaluation traces for GSM8K, MATH, AIME, and GPQA in \texttt{evaluation\_results} and compared model outputs from
M1 (Phi-4-mini) and M2 (Ministral-3B). Table~\ref{tab:qualitative} highlights representative cases where collaboration matters.
Two consistent patterns emerge. First, \textbf{complementary rescue}: M2 often resolves geometry or domain-knowledge questions that M1 misses (AIME/GPQA), while M1 is more reliable on compact algebraic simplifications (MATH). Second, \textbf{format sensitivity}: several incorrect outcomes are due to malformed or missing final answers (e.g., not listing comma-separated roots), which hurts conditional accuracy even when a partial derivation is plausible. This suggests that lightweight answer-normalization and stricter format prompts can further improve deployable accuracy without changing the underlying reasoning model.

At a finer granularity, we observe dataset-specific tendencies: (i) \textbf{GSM8K}: errors often stem from scale and formatting
slips (e.g., missing zeros) rather than algebraic reasoning; (ii) \textbf{AIME}: M2 more consistently handles geometric constraint
propagation and symmetry arguments, while M1 is more brittle when the solution depends on a specific configuration or auxiliary
construction; (iii) \textbf{MATH}: M1 does better on compact algebraic products and substitutions, but both models are sensitive to answer-format conventions (intervals, ordered pairs, and multi-root lists); (iv) \textbf{GPQA}: chemistry and physics questions
exhibit the most pronounced knowledge gaps, where both models can latch onto plausible distractors unless stoichiometry or
mechanistic details are explicit. These trends mirror the quantitative gains in Table~\ref{tab:table2}: collaboration helps most where the models’ error sets are complementary rather than identical.

\subsection{Limitations and Open Problems}
\label{sec:discussion_limits}
Cross-teaching activates only when at least one collaborator succeeds in \textsc{Round A}.
On extremely hard instances, both models fail and the method reduces to exploration/exploitation. Combining \methodname{} with retrieval, tool-use, or a verifier that can identify partially correct solutions could expand the teacher-available region. The reward structure extends naturally to $N{>}2$, but coordination raises questions:
which teacher to trust, how to allocate sampling budgets, and how to prevent group collapse onto a single dominant style.
Structured teacher selection and curriculum strategies (e.g., rotating teachers, difficulty-based routing) are promising directions.

\end{document}